\documentclass[12pt,journal,compsoc,onecolumn]{IEEEtran}

\usepackage[nocompress]{cite}
\usepackage[colorlinks,linkcolor=red, citecolor=blue]{hyperref}
\usepackage{graphicx}
\usepackage[cmex10]{amsmath}
\usepackage{amssymb}
\usepackage{amsthm}
\usepackage{bbm}
\usepackage{array}
\usepackage{mdwmath}
\usepackage{mdwtab}
\usepackage[font=normalsize,labelfont=sf,textfont=sf]{subfig}
\usepackage{fixltx2e}
\usepackage{stfloats}
\usepackage{url}
\usepackage{color}
\usepackage{hyperref}
\usepackage{lettrine}
\usepackage{algorithmicx,algorithm}
\usepackage{algpseudocode}
\usepackage[usenames,dvipsnames]{xcolor}
\usepackage{framed}
\usepackage{cancel}
\usepackage{setspace}
\usepackage{fancyhdr}
\usepackage[bottom,flushmargin,hang,multiple]{footmisc}
\DeclareMathOperator*{\argmin}{arg\,min}

\begin{document}
\newcommand{\todo}[1]{\textcolor{magenta}{[#1]}}
\IEEEoverridecommandlockouts 
% independent sign
\newcommand\independent{\protect\mathpalette{\protect\independenT}{\perp}}
\def\independenT#1#2{\mathrel{\rlap{$#1#2$}\mkern2mu{#1#2}}}

\newcommand{\adpact}{AdPaCT}
\newcommand{\admapl}{AMPL}

%%%%%%%%%%%%%%%%%%%%%%%%%%%%%%%%%%%%%
% Miscellaneous Commands 
%%%%%%%%%%%%%%%%%%%%%%%%%%%%%%%%%%%%%
\newcommand{\gcal}{\mathcal{G}}
\newcommand{\gcali}{\mathcal{G}^{i}}
\newcommand{\tcal}{\mathcal{T}}
\newcommand{\vcal}{\mathcal{V}}
\newcommand{\vcali}{\mathcal{V}^{i}}
\newcommand{\ecal}{\mathcal{E}}
\newcommand{\ecali}{\mathcal{E}^{i}}
\newcommand{\ind}{\mathbbm{1}}
\newcommand{\tsx}{\tilde{\sigma}_X}
\newtheorem{theorem}{Theorem}
\newtheorem{proposition}{Proposition}
\newtheorem{lemma}{Lemma}
\newtheorem{definition}{Definition}
\newtheorem{claim}{Claim}
\newtheorem{corollary}{Corollary}
\newtheorem{assumption}{Assumption}
\theoremstyle{remark}\newtheorem*{remark}{{\bf Remark}}

\newcommand{\vertiii}[1]{{\left\vert\kern-0.25ex\left\vert\kern-0.25ex\left\vert #1 
    \right\vert\kern-0.25ex\right\vert\kern-0.25ex\right\vert}}

\makeatletter
\def\footnoterule{\relax%
  \kern-5pt
  \hbox to \columnwidth{\hspace{0.2mm}\vrule width 0.5\columnwidth height 0.4pt\hfill}
  \kern4.6pt}
\makeatother
\title{Active Learning Algorithms for Graphical Model Selection\textsuperscript{\small\P}}

\author{\vspace{5mm}Gautam Dasarathy\textsuperscript{$\dagger$} \hspace{2mm} Aarti Singh\textsuperscript{$\ddagger$} \hspace{2mm} Maria Florina Balcan\textsuperscript{$\ast$}\\ \vspace{1mm}{\hfill Jong Hyuk Park\textsuperscript{$\flat$}\hfill}
%\\\vspace{2mm}${}^\dagger$ Machine Learning Department\\ ${}^\ast$ Computer Science Department\\ Carnegie Mellon University
}
\maketitle

%\aistatsaddress{ Unknown Institution }

\begin{abstract}
The problem of learning the structure of a high dimensional graphical model from data has received considerable attention in recent years. In many applications such as sensor networks and proteomics  it is often expensive to obtain samples from all the variables involved simultaneously. For instance, this might involve the synchronization of a large number of sensors or the tagging of a large number of proteins.  To address this important issue, we initiate the study of a novel graphical model selection problem, where the goal is to optimize the total number of scalar samples obtained by allowing the collection of samples from only subsets of the variables. We propose a general paradigm for graphical model selection where feedback is used to guide the sampling to high degree vertices, while obtaining only few samples from the ones with the low degrees. We instantiate this framework with two specific active learning algorithms, one of which makes mild assumptions but is computationally expensive, while the other is more computationally efficient but requires stronger (nevertheless standard) assumptions. Whereas the sample complexity of passive algorithms is typically a function of the maximum degree of the graph, we show that the sample complexity of our algorithms is provable smaller and that it depends on a novel local complexity measure that is akin to the average degree of the graph. We finally demonstrate the efficacy of our framework via simulations.

\end{abstract}

\renewcommand{\thefootnote}{$\P$}
\footnotetext[1]{\vspace{1mm}}
\renewcommand{\thefootnote}{$\dagger$}
\footnotetext[1]{Machine Learning Department, Carnegie Mellon University. \texttt{gautamd@cs.cmu.edu}}
\renewcommand{\thefootnote}{$\ddagger$}
\footnotetext[1]{Machine Learning Department, Carnegie Mellon University. Supported in part by NSF grants IIS-1247658, CAREER IIS-1252412, and AFOSR YIP FA9550-14-1-0285. \texttt{aarti@cs.cmu.edu}}
\renewcommand{\thefootnote}{$\ast$}
\footnotetext[1]{Machine Learning  and Computer Science Departments, Carnegie Mellon University. Supported in part by NSF grants NSF CCF-145117, NSF CCF-1422910, and a Sloan Fellowship. \texttt{ninamf@cs.cmu.edu}}
\renewcommand{\thefootnote}{$\flat$}
\footnotetext[1]{Computer Science Department, Carnegie Mellon University. \texttt{jp1@andrew.cmu.edu}}

\renewcommand*{\thefootnote}{\fnsymbol{footnote}}

\thispagestyle{plain}
\section{Introduction}
\label{sec:intro}

Probabilistic graphical models provide a powerful formalism for expressing the relationships among a large family of random variables. They are finding  applications in increasingly complex scenarios from computer vision and natural language processing to computational biology and statistical physics. One  important problem associated with graphical models is that of learning the structure of dependencies between the variables described by such a model from data. This is useful as it not only allows a succinct representation of a potentially complex multivariate distribution, but it might in fact reveal fundamental relationships among the underlying variables.  Unfortunately, the problem of learning the structure of graphical models is known to be hard in the high dimensional setting (where the number of observations is typically smaller than the total number of variables) since many natural sufficient statistics such as the sample covariance matrix are poorly behaved (see e.g., \cite{marvcenko1967distribution, johnstone2001distribution}). An exciting line of work has explored many conditions under which this problem becomes tractable (e.g., \cite{ravikumar2010high, ravikumar2011high, anandkumar2012ising, anandkumar2012gaussian}). Various authors have discovered that by constraining both the structure of the graph and the parameters of the probabilistic model, there are interesting situations where given $\mathcal{O}\left( \log p \right)$ samples from the underlying distribution, one can learn the structure and sometimes the parameters of the underlying graphical model. 

In this paper, we will look at this problem in a different light. In a wide variety of situations, it might be costly to obtain many samples across all the variables in the underlying system. For instance, in a sensor network, obtaining a sample across all the sensors is equivalent to obtaining a synchronized measurement from all the sensors. Similarly in many other applications in neuroscience \cite{keshri2013shotgun, turaga2013inferring, bishop2014deterministic} and proteomics \cite{sachs2005causal}, it might be much easier to obtain (marginalized) samples from small subsets of variables as opposed to full snapshots. 

We propose to handle this problem by what we call \emph{active marginalization}. That is, we present a general paradigm for ``activizing'' graphical model structure  learning algorithms. In the spirit of active learning, the algorithms we propose decide which vertices to marginalize out (and therefore which vertices to sample further from) based on the samples previously obtained. This  general framework is laid out in Algorithm~\ref{alg:meta}, which may be considered a meta algorithm that can be used to make any vertex based graph structure learning algorithm active. The algorithm principally uses two subroutines, one which selects candidate neighborhoods given a vertex and the other one which verifies if this candidate is a good one. We then instantiate this in two different ways for structure learning of Gaussian graphical models. Algorithm~\ref{alg:exhaustiveAlg}, which we call \adpact{} (for Adaptive Partial Correlation Testing) is an exhaustive search based algorithm that shows improved sample complexity over passive algorithms. However, like other algorithms based on exhaustive search, the computational complexity of this algorithm is often prohibitive. We then propose Algorithm~\ref{alg:LassoAlg}, which we call \admapl{} (for Adaptive Marginalization based Parallel Lasso). Algorithm~\ref{alg:LassoAlg} performs the neighborhood selection far more efficiently using the Lasso (as in the seminal work of \cite{meinshausen2006high}). At the cost of more restrictive, but standard assumptions, this gives us a much more computationally efficient  algorithm whose total sample complexity is similar to that of Algorithm~\ref{alg:exhaustiveAlg}, and therefore significantly better than its passive counterparts. 

While the sample complexity, i.e., the total number of samples needed for exact structure recovery with high probability, for typical passive graph learning algorithms scale as a function of the maximum degree $d_{\rm max}$ of the graph, the sample complexity of the algorithms we propose scale as a function of a more local property of the graph. For instance, in the Gaussian setting, the neighborhood selection algorithm of \cite{meinshausen2006high} is guaranteed to reconstruct the graph if the number of (scalar) measurements it obtains scales as $\mathcal{O}(d_{\rm max} p\log p)$ (i.e., $\mathcal{O}(d_{\rm max} \log p)$ samples each in $p$ dimensions). On the other hand, \adpact{} (Algorithm~\ref{alg:exhaustiveAlg}) and \admapl{}(Algorithm~\ref{alg:LassoAlg}) are guaranteed to work with $\mathcal{O}(\bar{d}_{\rm max} p \log p)$ samples, where $\bar{d}_{\rm max}$ is the average maximum degree across all the neighborhoods in the graph. We will define this quantity formally in Section~\ref{sec:prelim} and as we shall see in Section~\ref{sec:dbarmax}, $\bar{d}_{\rm max}$ can be significantly smaller than $d_{\rm max}$. 

\subsection{Related Work}
Active learning has been a major topic in recent years in machine
learning and an exhaustive literature survey is beyond the scope of this paper. We will point the reader, for instance, to \cite{balcan2006agnostic, dasgupta2007general, castro2008minimax, beygelzimer2010agnostic, hanneke2014theory}, for a taste of the  significant progress on general active learning algorithms with provable guarantees. It is worth noting that a majority of these results have been for supervised classification. To the best of our knowledge, the analysis of the power of active learning algorithms in the context of model selection for sparse graphical models is novel. A recent line of work \cite{arias2013fundamental, malloy2014near, haupt2009distilled} considers adaptive algorithms for recovering sparse vectors. As we shall see in the following sections, the algorithms we propose take advantage of the graph structure to outperform their passive counterparts. It is not immediately clear how the adaptive sparse recovery algorithms can be made to exploit such structure. There has also been some recent work on exploring the power of active learning for unsupervised learning tasks like clustering, e.g., \cite{eriksson2011active, krishnamurthy2012efficient, voevodski2012active,ailon2013breaking}, which is thematically closes to the goals of this paper. In fact, some of these algorithms may be modified from their current forms in order to yield active learning algorithms that can consistently learn tree-structured graphs. We also note here that the authors of \cite{vats2014active} consider a different notion of active graphical model selection. Their theoretical results   show the benefit of active learning for a specific family of graphs with two clusters; it does not appear to be straightforward to characterize the benefits of this algorithm in general. The present paper on the other hand, provides algorithms that outperform their classical passive counterparts both in theory and in experiments for more general sparse graphs. 

We would also like to point out related work on active learning for parameter and structure estimation in graphical models (c.f. \cite{tong2000active, tong2001active,murphy2001active}) that operates under the intervention model i.e. the query model allows obtaining conditional measurements where the values of some of the variables are fixed. In contrast, we explore the marginal query model where the samples are from marginal distributions, which is more natural in applications like sensor network inference and proteomics. The intervention model is particularly useful when the goal is causal structure discovery, while our paper does not endeavor to do this.

Finally, there is of course extensive work on passive algorithms for graphical model selection in high dimensions, see, e.g., \cite{meinshausen2006high, yuan2007model, ravikumar2011high, yuan2010high, ravikumar2010high,xue2012nonconcave,  kalisch2007estimating, anandkumar2012gaussian, anandkumar2012ising}. While some of these algorithms readily fit into the \emph{activization framework} of Algorithm~\ref{alg:meta}, creating active learning counterparts of other successful algorithms is an interesting avenue for future work.

\section{Problem Statement}
\label{sec:prelim}
Let $G = ([p],E)$ denote an undirected graph on the vertex set $[p]= \left\{ 1,2,\ldots,p \right\}$ with edge set $E \subseteq {p\choose 2}$. To each vertex of this graph, we associate the components of a  $0-$mean Gaussian\footnote{Our general framework applies much more broadly. We assume Gaussian distributions for ease of presentation.} random vector $X\in \mathbb{R}^p$ with covariance matrix $\Sigma\in \mathbb{R}^{p\times p}$. The density of $X$ is then given by 
\begin{equation}
f_X(x_1,\ldots x_p) = \frac{1}{\sqrt{(2 \pi)^p \left| K \right|}} \exp\left\{ -\frac{1}{2}x^T K x \right\},
\end{equation}
where $K = \Sigma^{-1}$ is the $p-$dimensional inverse covariance (concentration) matrix. We will abbreviate this density as $\mathcal{N}(0,\Sigma)$ in the sequel. 

$X$ is said to be Markov with respect to the graph $G$ if for any pair of vertices $i$ and $j$,  $\left\{ i,j \right\}\notin E$ implies that $X_i \independent X_j \mid X_{[p] \setminus \left\{ i,j \right\}}$. By the Hammersley-Clifford theorem \cite{lauritzen1996graphical}, we know that for all $\left\{ i,j \right\}\notin E$, $K_{ij} = 0$.

In this paper, we study the problem of recovering the structure of the graph $G$ given samples from the distribution $f_X$, that is, we would like to construct an estimate $\widehat{E}$ of the edge set of the underlying graph. As mentioned in Section~\ref{sec:intro}, we are interested in the setting where our estimators are allowed to \emph{actively marginalize} the components of $X$ and only observe samples from the desired components of the random vector $X$. More formally, operating in $L$ stages, the estimator or algorithm produces a sequence $\left\{ (S_k, n_k) \right\}_{k\in [L]}$, where $S_k\subseteq [p]$ and $n_k\in \mathbb{N}$. For each $k\in [L]$, the algorithm receives $n_k$ samples of the marginalized random vector $X_{S_k}\in \mathbb{R}^{\left| S_k \right|}$. Notice that $X_{S_k}$ is distributed according to $\mathcal{N}\left(0,\Sigma(S_k)\right)$, where $\Sigma(S_k)$ denotes the sub-block of the $\Sigma$ corresponding to the indices in $S_k$. The algorithm is also allowed to be \emph{adaptive} in that the choice of $(S_k, n_k)$ is allowed to depend on all the previous samples obtained. Notice that the ``passive'' algorithms (e.g., \cite{meinshausen2006high}, \cite{kalisch2007estimating}, \cite{ravikumar2011high}) can be thought of as operating in a single stage with $S_1 = [p]$. 

We will now define a natural metric for evaluating the performance of a graphical model selection algorithm that both allows us to compare the algorithms proposed here with their passive counterparts, and reflects the penalty for obtaining samples from large subsets of variables. Towards this end, first observe that the total number of scalar samples obtained by an algorithm as defined in the previous paragraph is given by $\sum_{k=1}^L \left| S_k \right|n_k$. Then, we may define the following.  

\begin{definition}[Total Sample Complexity]
\label{def:tsc}
Fix $\delta\in (0,1)$. Suppose that an algorithm returns an estimate $\widehat{E}_n$ given a budget of $n$ total scalar samples. We will say that its total sample complexity at confidence level $\delta$ is $n_0$ if for all $n \geq n_0$, $\mathbb{P}\left[ \widehat{E}_n \neq E \right]\leq 1-\delta$. 
\end{definition}

Our primary objective will be to produce active marginalization algorithms and demonstrate sufficient conditions on their total sample complexity (at some given level $\delta$). These sufficient conditions are expressed as a function of the graph size $p$, the magnitudes of the entries in the covariance matrix $\Sigma$, and an interesting structural property of the graph. To define this, we require a few more definitions. For a vertex $i\in [p]$, we will define its neighborhood ${N}_G(i) \triangleq \left\{ j\in [p]: \left\{ i,j \right\}\in E \right\}$ and the closure of its neighborhood $\overline{{N}}_G(i) \triangleq N_G(i) \cup \left\{ i \right\}$. The size of the set $N(i)$ will be referred to as the degree $d(i)$ of this vertex  \footnote{We will suppress the dependence on $G$ when it is clear from the context} and we define the \emph{maximum degree} $d_{\rm max} \triangleq \max_{i\in [p]} d(i)$. In addition to these standard definitions, we also define the \emph{local maximum degree} of a vertex as $d_{\rm max}^i \triangleq \max_{j\in \overline{N}(i)} d(j)$, that is, the maximum degree in $i$'s closed neighborhood. We also define the following quantity 
\begin{equation}
\bar{d}_{\rm max} \triangleq \frac{1}{p}\sum_{i=1}^p d^{i}_{\rm max},
\end{equation}
which is the average of the local maximum degrees. As we will see in the sequel, $\bar{d}_{\rm max}$, which is a more local notion of the complexity of the graph structure, will play a central role in the statement of our results concerning the total sample complexity of active learning algorithms for graphical model selection. 

First, we observe that in the passive setting, the total sample complexity (TSC) at some (constant) confidence level $\delta$ is  typically shown to be $\mathcal{O}(d_{\rm max} p \log p)$. For instance, using the results of \cite{wainwright2009sharp}, we can deduce that the Lasso based neighborhood selection method of \cite{meinshausen2006high} has a TSC which is no more than $\mathcal{O}(d_{\rm max} p \log p)$. Similarly, the PC algorithm \cite{spirtes2000causation}, which is a smart exhaustive search algorithm, has a TSC that scales like $\mathcal{O}(d_{\rm max} p \log p)$ (see \cite{kalisch2007estimating}). 

On the other hand, in Sections~\ref{sec:exhaustiveAlg}~and~\ref{sec:lassoAlgo} we will demonstrate algorithms whose total sample complexity (TSC) is bounded from above by $\mathcal{O}(\bar{d}_{\rm max} p \log p)$. By definition, $\bar{d}_{\rm max} \leq d_{\rm max}$, and it is not hard to see that $\bar{d}_{\rm max}$ can be significantly smaller than $d_{\rm max}$ when the graph has a heterogenous degree sequence, which is quite typical in many practical applications. 
Therefore the total sample complexity of our algorithms scales at a much better rate than their passive counterparts, which also results in a  significant practical advantage as shown by our preliminary simulation study in Section~\ref{sec:simulations}. 

In the wake of the significant progress made in the high-dimensional statistics literature, we can reason about these results intuitively. To discover the neighborhood of a vertex $i\in [p]$, it should suffice to sample $i$ and (at least) its neighbors $\mathcal{O}(d_i \log p)$. This implies that the neighbor of $i$ with maximum degree in $\bar{N}(i)$ will force $i$ to be sampled $\mathcal{O}(d^i_{\rm max}\log p)$ times. Therefore, a total of $\mathcal{O}(\bar{d}_{\rm max} p \log p)$ samples should suffice, if the algorithm is able to focus its samples on the right subsets of vertices. We demonstrate in Sections~\ref{sec:exhaustiveAlg}~and~\ref{sec:lassoAlgo} exactly how this can be achieved.  It is also curious to note that $\bar{d}_{\rm max}$ appears as the natural notion of local complexity of the graph rather than, say, the average degree. Understanding whether this quantity is fundamental to active learning of graphical models, via a minimax lower bound, for instance, is an extremely interesting avenue for future work. 

Formal lower bound arguments for passive algorithms for recovering the 
structure of Gaussian graphical models exist in \cite{wang2010information}, 
which establishes that for the class of graphs with maximum degree 
$d_{\max}$, if the smallest partial correlation coefficient between any 
pair of nodes conditioned on any subset of nodes is bounded from below 
by a constant, then the passive algorithms require $\Omega(d_{\max} p \log 
p)$ scalar observations. The class of graphs with average maximum degree 
$\bar{d}_{\max}$ is a subset of the class of graphs with maximum degree 
$d_{\max}$, implying that the lower bound may not apply to the former. 
However, a careful investigation of the lower bound construction in 
\cite{wang2010information} reveals that the lower bound is based on hardest 
examples from the class of graphs with a single clique of size $d_{\max}$ 
and all other nodes disconnected, for which $\bar{d}_{\max} = d_{\max}^2/p 
\leq d_{\max}$. This establishes that, for the class of graphs whose $d_{\rm max}$ and $\bar{d}_{\rm max}$ are such that $\bar{d}_{\max} \leq d_{\max}^2/p$, all passive 
algorithms are much worse than the active algorithms proposed in this paper. 

\subsection{$\bar{d}_{\rm max}$ versus $d_{\rm max}$}
\label{sec:dbarmax}
As mentioned earlier, the total sample complexity of our algorithms scale like $\mathcal{O}(\bar{d}_{\rm max} p\log p)$ as opposed to the typical  $\mathcal{O}(d_{\rm max} p\log p)$  scaling of passive algorithms. By definition, we can see that $\bar{d}_{\rm max} \leq d_{\rm max}$. There are many situations where $\bar{d}_{\rm max}$ can be significantly smaller than $d_{\rm max}$, potentially allowing the active algorithms we propose to yield steep savings in terms of the total sample complexity by running the active algorithm. As is often the case in real-world graphs, if the degree distribution is non-uniform, it is likely the case that $\bar{d}_{\rm max}$ is significantly smaller than $d_{\rm max}$. For instance, consider a graph on $p$ vertices where there are $\Theta(p)$ low (say $\Theta(1)$) degree vertices, and there are very few (say $\Theta(1)$) vertices with degree $d$. For such a graph, observe that $\bar{d}_{\rm max}$ scales like $d^2/p$, while $d_{\rm max}$ scales like $d$.

Another interesting case is that of graphs with small \emph{correlation dimension} \cite{chan2007approximation}, which can be thought of as a global version of the doubling dimension (see, e.g., \cite{slivkins2007distance} for more on doubling dimension). $\kappa$ is said to be the correlation dimension of $G = ([p],E)$ if $\kappa$ is the smallest constant such that for all $r\in [p]$:  $\sum_{i\in [p]} \left| B(i,2r) \right|\leq 2^\kappa \sum_{i\in [p]} \left| B(i,r) \right|$, where $B(i,r)$ is the set of all vertices that are at most $r$ away (in shortest path distance) from $i$. Such graphs are of interest since they, like graphs with small  doubling dimension, are amenable to more efficient graph processing algorithms (see \cite{chan2007approximation} and references there in). Suppose a graph has a correlation dimension of $\kappa$, then it is not hard to see that $\bar{d}_{\rm max} \leq 2^\kappa \sum_{i} d_i/p$, i.e., $\bar{d}_{\rm max}$ is controlled by the average degree of the graph, which can be significantly smaller than $d_{\rm max}$.

\section{A GENERAL FRAMEWORK FOR ACTIVE GRAPHICAL MODEL SELECTION}
\label{sec:metaAlgo}

In this section, we will first describe our general framework for building an active algorithm for graphical model selection; this is described in Algorithm~\ref{alg:meta}. We note here that the ideas and results here are not predicated upon the assumption that $X$ is Gaussian. 
%%%%%%%%%%%%%%%%%%%%%%%%%%%%%%%%%%%%%%%%%%%%%%%%%%%%%%%%%%%%%%
\algnewcommand{\LineComment}[1]{\Statex \(/*\) {\bf #1} \(*/\)}
\begin{algorithm*}[]
\caption{Active Neighborhood Selection}
\label{alg:meta}
\begin{algorithmic}[1]
\Require{budget  $B\in \mathbb{N}$, sample complexity functions $g(), h()$}, subroutines {\tt\bf nbdSelect, nbdVerify}. 
\State Set $\ell = 1$, ${\rm nSamples} = 0$, and initialize $\widehat{N}(i)$, $\forall i\in [p]$, \textsc{ nbdFound, settled},  $S_1$, $S_2$ to $\emptyset$ (the empty set). 
\vspace{2mm}
\LineComment{Obtain new samples from {\sc settled}$^c$}
\vspace{2mm}
\Repeat 
%\State $\ell = \ell + 1$
\State Obtain $g(\ell)$ independent samples $X_{\mbox{\sc settled}^c}^{(j)}$,$j = 1,\ldots, g(\ell)$; add to $S_1$
\State Obtain $h(\ell)$ independent samples $X_{\mbox{\sc settled}^c}^{(j)},$ $j = 1,\ldots, h(\ell)$; add to $S_2$
\State Increment ${\rm nSamples}$ by $\left( p - \left| \mbox{\small\sc settled} \right|\right)$ $\times \left( g(\ell)  + h(\ell) \right)$.
\vspace{2mm}
\LineComment{Generate and verify candidate neighborhoods}
\vspace{2mm}
\For{$i\in $ {\sc nbdFound}$^c$}
\State{$\widehat{N}(i) = $ {\tt\bf nbdSelect}$(i,\ell, \{X_{\mbox{\sc settled}^c}^{(j)}\}_{j\in S_1})$} \qquad\;\;\qquad $/*$ $|\widehat{N}(i)|$ is at most $\ell$ $*/$
\If{({\tt\bf nbdVerify}$(i,\widehat{N}(i), \{X_{\mbox{\sc settled}^c}^{(j)}\}_{j\in S_2})$ = {\bf true})}
\State {\sc nbdFound} $=$ {\sc nbdFound} $\cup \left\{ i \right\}$
\EndIf
\EndFor
\vspace{2mm}
\LineComment{Settle vertices}
\vspace{2mm}
\For {$i \in $ {\sc nbdFound}}
\If{$\widehat{N}(i) \subseteq $ {\sc nbdFound}} {\sc settled} $=$ {\sc settled} $\cup \left\{ i \right\}$ \EndIf
\EndFor
\State Set $\ell = 2\times \ell$, $S_1,S_2 = \emptyset$.
\Until {$\ell \geq 2p$ or {\sc nbdFound} $=[p]$ or ${\rm nSamples}> B$}
\State \Return{Graph $\widehat{G}$ such that $\left\{ i,j \right\}\in \widehat{G}\Leftrightarrow i\in \widehat{N}(j)$ or $j\in \widehat{N}(i)$}
\end{algorithmic}
\end{algorithm*}
%%%%%%%%%%%%%%%%%%%%%%%%%%%%%%%%%%%%%%%%%%%%%%%%%%%%%%%%%%%%%%

Algorithm~\ref{alg:meta} accepts a natural number $B$ which is a budget for the total number of scalar samples allowed. In our theoretical analysis, we will identify a sufficient condition for the budget in terms of natural parameters of the graphs we wish to learn. Algorithm~\ref{alg:meta} also accepts ``sample complexity functions'' $g, f: \mathbb{N}\to \mathbb{N}$. Finally, the algorithm depends on two subroutines: {\tt\bf nbdSelect}() and {\tt\bf nbdVerify}(). The former subroutine takes as input (the index of) a vertex $i$, a candidate neighborhood size $\ell$, and samples
%\footnote{At step $\ell$, note that there are $g(\ell)$ samples available by telescoping cancellation.} 
from a subset $S$ of the variables. It then return a subset of $S$ of size no more than $\ell$ as its estimate $\widehat{N}(i)$ of the neighborhood of $i$. {\tt\bf nbdVerify()}  accepts a vertex $i$, a candidate neighborhood $\widehat{N}(i)$, and samples from the variables in {\sc settled}$^c$ $= [p]\setminus ${\sc settled}. It then checks if $\widehat{N}(i)$ is indeed a potential neighborhood of $i$.  We will refer to $g()$ and $h()$ as the sample complexity functions of the subroutines {\tt\bf nbdSelect} and {\tt\bf nbdVerify} respectively. At this point, we will let  these subroutines and their sample complexity functions remain abstract and focus on the structural details of our active graphical model selection framework. Sections~\ref{sec:exhaustiveAlg} and \ref{sec:lassoAlgo} will demonstrate two explicit  instantiations of this framework when $X$ is Gaussian, and the application to more general distributions would simply involve establishing appropriate instantiations of {\tt\bf nbdSelect}(),  {\tt\bf nbdVerify}(), $g()$, and $h()$. 

We will now describe Algorithm~\ref{alg:meta}. We start with an empty graph on $[p]$, i.e., $\widehat{N}(i) = \emptyset$ for all $i\in [p]$. And initialize the counter $\ell$ to 1,  and the variable {\rm nSamples} to $0$. We also initialize the sets {\sc nbdFound, settled} to $\emptyset$. As the names suggest, {\sc nbdFound} will be used to keep track of the vertices whose neighborhood estimates the algorithm is confident about and {\sc settled} keeps track of the vertices that no longer need to be sampled from. 
Notice that the faster {\sc settled} is populated, the better the performance is of Algorithm~\ref{alg:meta} in terms of the total sample complexity, since in successive stages only the vertices in {\sc settled}$^c$ are sampled.  The algorithm then loops over $\ell$ (by doubling it) until one of the following holds: $\ell>2p$, {\sc nbdFound} $= [p]$ or nSamples exceeds the budget. At each iteration, the algorithm obtains $g(\ell)+ h(\ell)$ new samples from variables in the set {\sc settled}$^c$ and stores these in $S_1$ and $S_2$ as in Steps 3 and 4. Next, in Steps 6-11, for each vertex $i\notin $ {\sc settled} the algorithm uses the subroutine {\tt\bf\bf nbdSelect}() to estimate a neighborhood of $i$ of size at most $\ell$. And, then if this neighborhood passes the check of the subroutine {\tt\bf nbdVerify}, the algorithm adds $i$ to the set {\tt\bf nbdFound}. Finally, in Steps 12 - 15, the set {\sc settled} gets updated. Any $i$ in {\sc nbdFound} whose  entire estimated neighborhood is in {\sc nbdFound} gets enrolled in {\sc settled} and does not get sampled henceforth. That is, the algorithm ``settles'' a vertex $i\in [p]$ if it is both confident about the vertex's neighborhood and about the neighborhood of $i$'s neighbors. It is this step that gives our algorithm its improved total sample complexity. 

We are now ready to state the following result that characterizes the performance of Algorithm~\ref{alg:meta}. We will postpone the proof, which formalizes the intuition of the above marginalization argument, to Appendix~\ref{sec:proofOfTheorem1}.

\begin{theorem}
\label{thm:metaAlg}
Fix $\delta \in (0,1)$. For each $\ell \leq d_{\rm max}$, assume that the subroutines {\tt\bf nbdSelect} and {\tt\bf nbdVerify} satisfy: 
\begin{enumerate}
\item[(C1)] For any vertex $i\in [p]$ and subset $F\subseteq [p]$ that are such that $\left| N(i) \right| = d_i \leq \ell$ and $\overline{N}(i) \subseteq F$, the following holds. Given $g(\ell)$ samples from $X_F$, {\tt\bf nbdSelect}$(i, \ell, \left\{ X_F^{(j)} \right\}_{j\in S_1})$ returns the true neighborhood of $i$ with probability greater than $1-\delta/2p d_{\rm max}$. 
\item[(C2)] For any vertex $i\in [p]$ and subsets $F, H\subseteq [p]$ that are such that $\left| N(i) \right| = d_i \leq \ell$, $\overline{N}(i) \subseteq F$, and $H\subseteq F$, the following holds. Given $h(\left| H \right|)$ samples from $X_F$, {\tt\bf nbdVerify} $(i,H,\left\{ X_F^{(j)} \right\})$ returns {\bf true} if and only if $N(i)\subseteq H$ with probability greater than $1-\delta/2p d_{\rm max}$. 
\end{enumerate}
Then, with probability no less than $1-\delta$, Algorithm~\ref{alg:meta} returns the correct graph. Furthermore, it suffices if $B\geq \sum_{i\in [p]}\sum_{0\leq k \leq \lceil \log_2d_{\rm max}^i \rceil} g(2^k)  + h(2^k))$. That is, Algorithm~\ref{alg:meta} has a total sample complexity of $\sum_{i\in [p]}\sum_{0\leq k \leq \lceil \log_2d_{\rm max}^i \rceil} g(2^k)  + h(2^k)$ at confidence level $1-\delta$. 
\end{theorem}

Theorem~\ref{alg:meta} tells us that if the subroutines {\tt\bf nbdSelect} (resp. {\tt\bf nbdVerify}) satisfies condition (C1) (resp. (C2)) with probability exceeding $1 - \delta/2p d_{\rm max}$ for a fixed $\ell$ with $g(\ell)$ (resp. $h(\ell)$) samples, then Algorithm~\ref{alg:meta} has a TSC of $\sum_{i\in [p]}\sum_{k \leq \lceil \log_2d_{\rm max}^i \rceil} g(2^k)  + h(2^k)$ at a confidence level $ 1- \delta$. In what follows, we will show two specializations of Algorithm~\ref{alg:meta}. Our strategy to establish the TSC of these algorithms will be to first estimate the probability that {\tt\bf nbdSelect} and {\tt\bf nbdFound} fail to satisfy (C1) and (C2). This will suggest a choice for the functions $g()$ and $h()$, which can then be used to identify the total sample complexity.

\section{The AdPaCT Algorithm}
\label{sec:exhaustiveAlg}

We call our first instantiation of Algorithm~\ref{alg:meta}, the AdPaCT algorithm, which stands for \underline{A}daptive \underline{Pa}rtial \underline{C}orrelation \underline{T}esting. In this section, for the sake of simplicity and concreteness, we will assume that $X$ follows the $0$ mean multivariate normal distribution with a covariance matrix $\Sigma \in \mathbb{R}^{p\times p}$.  

First, we observe that since $X\sim \mathcal{N}\left( 0, \Sigma \right)$, the partial correlation coefficient contains all the conditional independence information of the distribution. In particular, for a pair of vertices $i, j \in [p]$, and for a subset $S \subseteq [p]$, the partial correlation coefficient $\rho_{ij|S}= 0$ if and only if  $X_i \independent X_j \mid X_S$ (except for a pathological, measure $0$ set of covariance matrices). As the name suggests, the \adpact{} algorithm uses estimates of these partial correlation coefficients in order to learn the graph represented by $\Sigma$. Recall that the partial correlation coefficient $\rho_{ij\mid S}$ satisfies the following recursive relationship which holds for any $k\in S$: 
\begin{equation}
{\rho}_{i,j\mid S} = \frac{{\rho}_{i,j\mid S\setminus\left\{ k \right\}} - {\rho}_{i,k\mid S\setminus \left\{ k \right\}}{\rho}_{j,k\mid S\setminus \left\{ k \right\}}}{\sqrt{\left( 1 - {\rho}^2_{i,k\mid S\setminus \left\{ k \right\}} \right)\left(1 -  {\rho}^2_{j,k\mid S\setminus \left\{ k \right\}} \right)}}. \label{eq:partialCorrelationRecursion}
\end{equation}
\hypertarget{para:empiricalPartialCorrelation}{In order to compute empirical estimates of these quantities one can begin with the natural empirical estimates of the correlation coefficients and substitute recursively in the above formula, or equivalently, one might invert the relevant sub-matrices of the empirical covariance matrix of the observed data. In what follows, we will write $\widehat{\rho}_{ij\mid S}$ to mean either of these estimates; our theoretical results hold for both.}  

We will now describe Algorithm~\ref{alg:exhaustiveAlg}.  The framework provided by Algorithm~\ref{alg:meta} will allow us to do this by prescribing choices for the subroutines {\tt\bf nbdSelect} and {\tt\bf nbdVerify} and the functions $g()$ and $h()$.  

\begin{algorithm}[h]
\caption{\adpact{}: Adaptive Partial Correlation Testing}
\label{alg:exhaustiveAlg}
\begin{algorithmic}[1]
\Require{Budget $B\in \mathbb{N}$, a constant $c>0$, and threshold $\xi>0$.}
\vspace{3mm}
\Statex \hspace{-6mm}\underline{Sample Complexity Functions}
\State $g(\ell) = \lceil c \ell \log p \rceil$, $h(\ell) = 0$.
\vspace{3mm}
\Statex \hspace{-6mm}\underline{{\tt\bf nbdSelect}$(i,\ell, \{X_F^{(j)}\}_{j\in S_1})$}
\For {$k = \ell/2+1, \ell/2 + 2, \ldots, \ell$}\Comment {ensure that $\ell$ is a power of 2}
\State $\mathcal{S} = \left\{ S\subseteq F : \left| S \right| = k, \max_{j\notin S}\left| \widehat{\rho}_{i,j\mid S} \right|\leq \xi \right\}$
\If{$\mathcal{S} = \emptyset$} {\bf continue } (i.e., go to Step 7)
\Else{ \Return{$\widehat{S} = \arg\min_{S\in \mathcal{S}} \max_{j\notin S}\left| \widehat{\rho}_{i,j\mid S} \right|$}}
\EndIf
\EndFor
\State \Return {$\emptyset$}
\vspace{3mm}
\Statex \hspace{-6mm}\underline{{\tt\bf nbdVerify}$(i,S, \{X_F^{(j)}\}_{j\in S_2})$}
\If{ $S\neq \emptyset$ } \Return {\bf true} \Else { \Return{\bf false}}
\EndIf
\end{algorithmic}
\end{algorithm}

We choose the sample complexity functions $g(\ell) = \lceil c \ell \log p \rceil$ and $h(\ell) = 0$. The {\tt\bf nbdSelect} subroutine exhaustively searches over all subsets $S\subseteq F$ (note that $F$ will be $[p]\setminus${\sc settled} when {\tt\bf nbdSelect} is called inside of Algorithm~\ref{alg:meta}) of cardinality between $\ell/2 +1$ and $\ell$ to find the smallest set $\widehat{S}$ which is such that  $\max_{j\notin \widehat{S}}\left| \widehat{\rho}_{i,j\mid \widehat{S}} \right|\leq \xi$. If such a set is not found, then the subroutine returns the empty set. 

Observe that {\tt\bf nbdSelect} on its own performs conditional independence tests to ensure that it is returning the right neighborhood. Therefore, {\tt\bf nbdVerify} simply returns {\bf true} unless $S = \emptyset$. It is worth observing here that the passive counterpart of this algorithm is a natural algorithm for Gaussian graphical model selection and indeed serves as the foundation for various algorithms in literature like the CCT algorithm of  \cite{anandkumar2012gaussian} and the PC algorithm (see e.g., \cite{spirtes2000causation}).

In order to theoretically characterize the performance of Algorithm~\ref{alg:exhaustiveAlg}, we need the following assumptions. 
\vspace{1mm}
\begin{enumerate}
\item[(A1)] The distribution of $X$ is faithful to $G$. 
\vspace{2mm}
\item[(A2)] For each $i,j\in [p]$ and $S\subseteq [p]\setminus \left\{ i,j \right\}$, $\left|\rho_{i,j\mid S}\right|\leq M$. And, for each $i,j\in [p]$ and $S\subseteq [p]$, if $X_i\not\independent X_j\mid X_S$, then $\left|\rho_{i,j\mid S}\right|\geq m$. 
\end{enumerate}
\vspace{1mm}

Assumption (A1) is a standard assumption in the graphical model selection literature and is violated only on a set of measure 0 (see e.g., \cite{spirtes2000causation, pearl2000causality}). Assumption (A2) has appeared in the literature (e.g., \cite{kalisch2007estimating}) as a way of strengthening the faithfulness assumption.  While the upper bound in assumption (A2) is a mild  regularity condition, the lower bound of (A2) may be hard to verify in practice. However, under certain parametric and structural conditions, one can obtain a handle on $m$. For example, the authors in \cite{anandkumar2012gaussian} show that if the underlying graph has small local separators and if the concentration matrix is \emph{walk-summable}, then $m$ in (A2) can be replaced essentially by  the smallest non-zero entry of the  concentration matrix. 

We can now state the following theorem about the performance of the AdPaCT algorithm. 

\begin{theorem}
\label{thm:exhaustiveAlg}
Fix $\delta\in (0,1)$ and suppose that assumptions (A1) and (A2) hold. Then, there exists a constant $c = c(m, M, \delta)$ such that if we set $g(\ell) = \lceil c \ell \log p \rceil$ and $\xi = m/2$, then with probability no less than $1  - \delta$, the following hold: 
\begin{enumerate}
\item The AdPaCT algorithm successfully recovers the graph $G$.
\item The computational complexity of the AdPaCT algorithm is no worse than $\mathcal{O}(p^{d_{\rm max} + 2})$
\end{enumerate}
This implies that the total sample complexity of the AdPaCT algorithm at confidence level $1 - \delta$ is bounded by $2c \bar{d}_{\rm max} p\log p$. 
\end{theorem}

To prove this theorem, as mentioned earlier, we show that the choice for the subroutines {\tt\bf nbdSelect} and {\tt\bf nbdVerify} satisfy the conditions (C1) and (C2) of Theorem~\ref{thm:metaAlg} with high probability. We bound the event that {\tt\bf nbdSelect} fails in some iteration $\ell$ in terms of concentration inequalities for the partial correlation coefficient. This gives us a corresponding choice of $g(\ell)$ that determines the TSC of the AdPaCT algorithm. Please refer to  Appendix~\ref{sec:exhaustiveAlgProof} for the details of the proof.

It is clear that this procedure is advantageous over passive algorithms in situations where $d_{\rm max}$ is large compared to $\bar{d}_{\rm max}$. Unfortunately, in these settings which lend themselves to improved sample complexity, the computational complexity of the AdPaCT algorithm could be prohibitively large. In the next section, we will propose a different instantiation of Algorithm~\ref{alg:meta} based on the Lasso \cite{tibshirani1996regression} which achieves sample complexity savings similar to Algorithm~\ref{alg:exhaustiveAlg} whilst also enjoying vastly lower computational complexity. 
\begin{algorithm}[h]
\caption{\admapl{}: Adaptive Marginalization based Parallel Lasso}
\label{alg:LassoAlg}
\begin{algorithmic}[1]
\Require{Budget $B\in \mathbb{N}$, a constant $c>0$, and threshold $\xi>0$.}
\vspace{3mm}
\Statex \hspace{-6mm}\underline{Sample Complexity Functions}
\State $g(\ell) = h(\ell) = \lceil \ell c \log p\rceil$.
\vspace{3mm}
\Statex \hspace{-6mm}\underline{{\tt\bf nbdSelect}$(i,\ell, \{X_F^{(j)}\}_{j\in S_1})$}
\State Let $y\in \mathbb{R}^{\lceil\ell c\log p\rceil}$ be the vector of samples from $X_i$ in $\mathcal{S}_1$. 
\State Let $\mathbf{X}\in \mathbb{R}^{(\lceil \ell c\log p\rceil)\times (p-\left| \mbox{\sc settled} \right| -1)}$ be the corresponding matrix of samples from $X_{[p]\setminus\{\mbox{\sc settled}\cup \left\{ i \right\}\}}$. 
\State $\widehat{\beta}\leftarrow \mbox{\sc Lasso}(y, \mathbf{X})$
\If{ $\left| {\rm supp}(\widehat{\beta}) \right|> \ell$} \Return { top $\ell$ coordinates of $\left| \widehat{\beta} \right|$ }
\Else { \Return $\widehat{\beta}$}
\EndIf
\vspace{3mm}
\Statex \hspace{-6mm}\underline{{\tt\bf nbdVerify}$(i,S, \{X_F^{(j)}\}_{j\in S_2})$}
\If{ for each $j\in [p]\setminus F\cup S\cup \{i\}$, $\left|\widehat{\rho}_{i,j\mid S}\right|\leq \xi$ } \Return {\bf true} \Else { \Return{\bf false}}
\EndIf
\end{algorithmic}
\end{algorithm}

\section{The \admapl{} algorithm}
\label{sec:lassoAlgo}

In this section, we will discuss a computationally efficient active marginalization algorithm for learning graphs. As alluded to in Section~\ref{sec:exhaustiveAlg}, this algorithm uses Lasso as an efficient means for neighborhood selection and hence can be thought of as an active version of the seminal work of \cite{meinshausen2006high}. We call this algorithm \admapl{} for \underline{A}daptive \underline{M}arginalization based \underline{P}arallel \underline{L}asso.  

As in Section~\ref{sec:exhaustiveAlg}, we will describe the algorithm by prescribing choices for the subroutines {\tt\bf nbdSelect} and {\tt\bf nbdVerify} and the functions $g()$ and $h()$.

We choose the sample complexity functions to be $g(\ell) = h(\ell) = c\ell \log p$. {\tt\bf nbdSelect} operates as follows. Let $y$ denote the vector of samples corresponding to the random variable $X_i$ and let $\mathbf{X}$ denote the corresponding matrix of  samples from the the random variables $X_{F\setminus \left\{ i \right\}}$. The subroutine solves the following optimization program 
\begin{equation}
\widehat{\beta} = \underset{\beta\in \mathbb{R}^{\left| F \right|-1}}{\arg\min} \frac{1}{2 n_{i,\ell}}\left\| y - \mathbf{X}\beta \right\|_2^2 + \lambda_{\ell,i} \left\| \beta \right\|_1,
\end{equation}
where the choice of $\lambda_{\ell,i}$ is stated in Theorem~\ref{thm:lassoAlg} and $n_{i,\ell}$ is the number of samples (i.e., dimension of $y$). If the size of the support of $\widehat{\beta}$ is greater than $\ell$, then the algorithm returns the $\ell$ largest coordinates of $\left|\widehat{\beta}\right|$, else the algorithm returns the support of $\widehat{\beta}$. {\tt\bf nbdVerify} returns {\bf true} if $\left|\widehat{\rho}_{i,j\mid \widehat{N}(i)}\right|\leq \xi$ for every $j\in [p]\setminus(\mbox{\sc settled}\cup \widehat{N}(i)\cup\left\{ i \right\})$, else it returns {\bf false}.

Before we can state our theoretical results on the performance of the AMPL algorithm, we need to make some assumptions. For each $i\in [p]$, let ${\Sigma}^{\setminus i}$ denote the covariance matrix with the $i-$the row and column removed and set $S_i \triangleq {N}(i)$. We will assume the following conditions. 

\begin{enumerate}
\item[(A3)] There exists a constant $\gamma\in (0,1]$ such that for all $i\in [p]$,  $\widetilde{\Sigma}^i$ satisfies the following 
\begin{equation}
\vertiii{\widetilde{\Sigma}^i_{S^c_iS_i}\left(\widetilde{\Sigma}^i_{S_iS_i}\right)^{-1}}_\infty\leq 1 - \gamma
\end{equation}
\item[(A4)] For all $i\in [p]$, the covariance matrix $\widetilde{\Sigma}^i$ also satisfies 
\begin{align}
\Lambda_{\rm min}\left(\widetilde{\Sigma}^i_{S_iS_i}\right) &\geq C_{\rm min} > 0\\
\Lambda_{\rm max}\left(\widetilde{\Sigma}^i_{S_iS_i}\right) &\leq C_{\rm max} < +\infty
\end{align}
\end{enumerate}

Assumption (A3) is a kind of incoherence assumption often dubbed the \emph{irrepresentability condition}; similar assumptions have appeared in the literature for graphical model selection \cite{ravikumar2011high, ravikumar2010high, meinshausen2006high} and in the analysis of the Lasso \cite{van2009conditions}. Intuitively speaking, this restricts the influence that non-edge pairs of vertices have on the pairs of vertices that are edges. 

Assumption (A4) is a commonly imposed regularity  condition on the covariance matrix. 

We can now state the following theorem that characterizes the performance of the AMPL algorithm. 

\begin{theorem}
\label{thm:lassoAlg}
Fix $\delta >0$. Suppose that assumptions (A1)-(A4) hold. There exists constants $C_1, C_2, C_3$ which depend on $\Sigma, m, \delta$ such that if we set $c = C_1$ (i.e., $g(\ell) = c \ell \log p$), $\xi = m/2$, $\lambda_{\ell} = \sqrt{\frac{2 C_2 \left\| \Sigma \right\|_\infty}{C_1 \gamma^2}}$, and budget  $B= 2c\bar{d}_{\rm max} p\log p$, then with probability at least $1 - \delta$, the following hold
\begin{enumerate}
\item the AMPL algorithm successfully recovers the graph $G$, 
\item The computational complexity is bounded from above by $d_{\rm max}p \mathfrak{C}$, where $\mathfrak{C}$ is the computational cost of solving a single instance of Lasso,
\end{enumerate}
provided $m \geq \left( \frac{C_{\rm min}}{C_{\rm max}} + \frac{C_{\rm max}}{C_{\rm min}} + 2 \right)\times \frac{1}{4 \min_i \left| \Sigma_{ii} \right|} \left[ C_3 \sqrt{\frac{2C_1 \left\| \Sigma \right\|_\infty}{C_2 \gamma^2}}\max_{i} \vertiii{\left(\widetilde{\Sigma}^i_{N(i),N(i)}\right)^{-1/2}}^2_\infty + 20 \sqrt{\frac{\left\| \Sigma \right\|_\infty}{C_{\rm min} C_2}}\right]$.
 \end{theorem}

Again, we prove this theorem by showing that our choice for the subroutines {\tt\bf nbdSelect} and {\tt\bf nbdVerify} satisfy the conditions (C1) and (C2). The proof characterizing the behavior of {\tt\bf nbdVerify} is very similar to the proof of Theorem~\ref{thm:exhaustiveAlg}. On the other hand, to characterize the behavior of {\tt\bf nbdSelect}, one part of our reasoning is similar to the argument in \cite[Theorem 3]{wainwright2009sharp}. The rest of the proof follows from a strengthening of the argument in  \cite{wainwright2009sharp} for the case when the degree is $o(\log p)$. We also needed to be cognizant of the fact that our adaptive marginalization approach results in samples in different stages having different distributions.  We refer the interested reader to Appendix~\ref{sec:LassoAlgProof} for the details.

\section{SIMULATIONS}
\label{sec:simulations}
\begin{figure*}
\centering
\begin{minipage}[b]{0.4\textwidth}
\includegraphics[scale = 0.41]{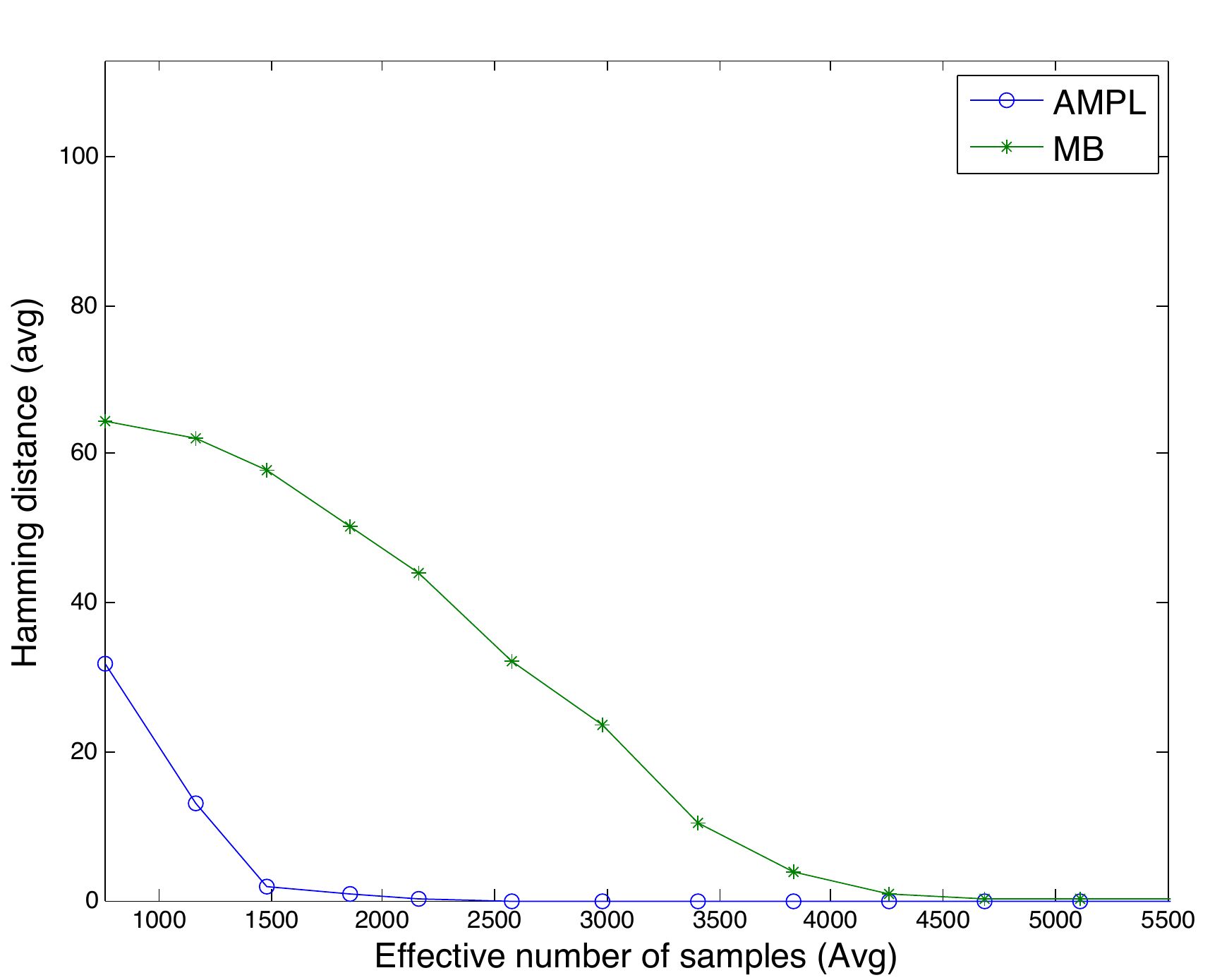}
\caption*{{Single Clique}}
\label{fig:singleClique}
\end{minipage}
\begin{minipage}[b]{0.4\textwidth}
\includegraphics[scale = 0.41]{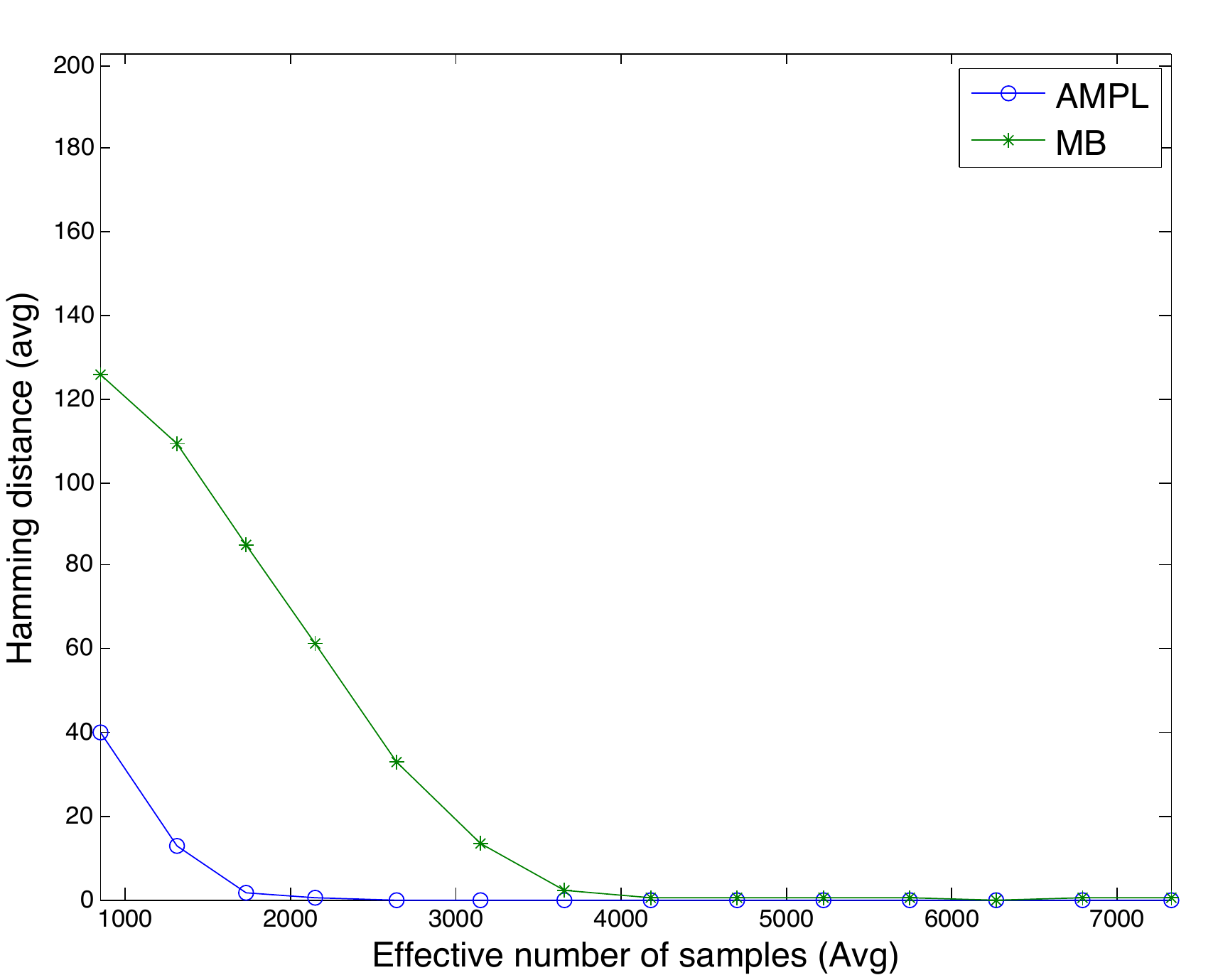}
\caption*{{Multiple Cliques}}
\label{fig:singleClique}
\end{minipage}
\vspace{4mm}
\begin{minipage}[b]{0.4\textwidth}
\includegraphics[scale = 0.41]{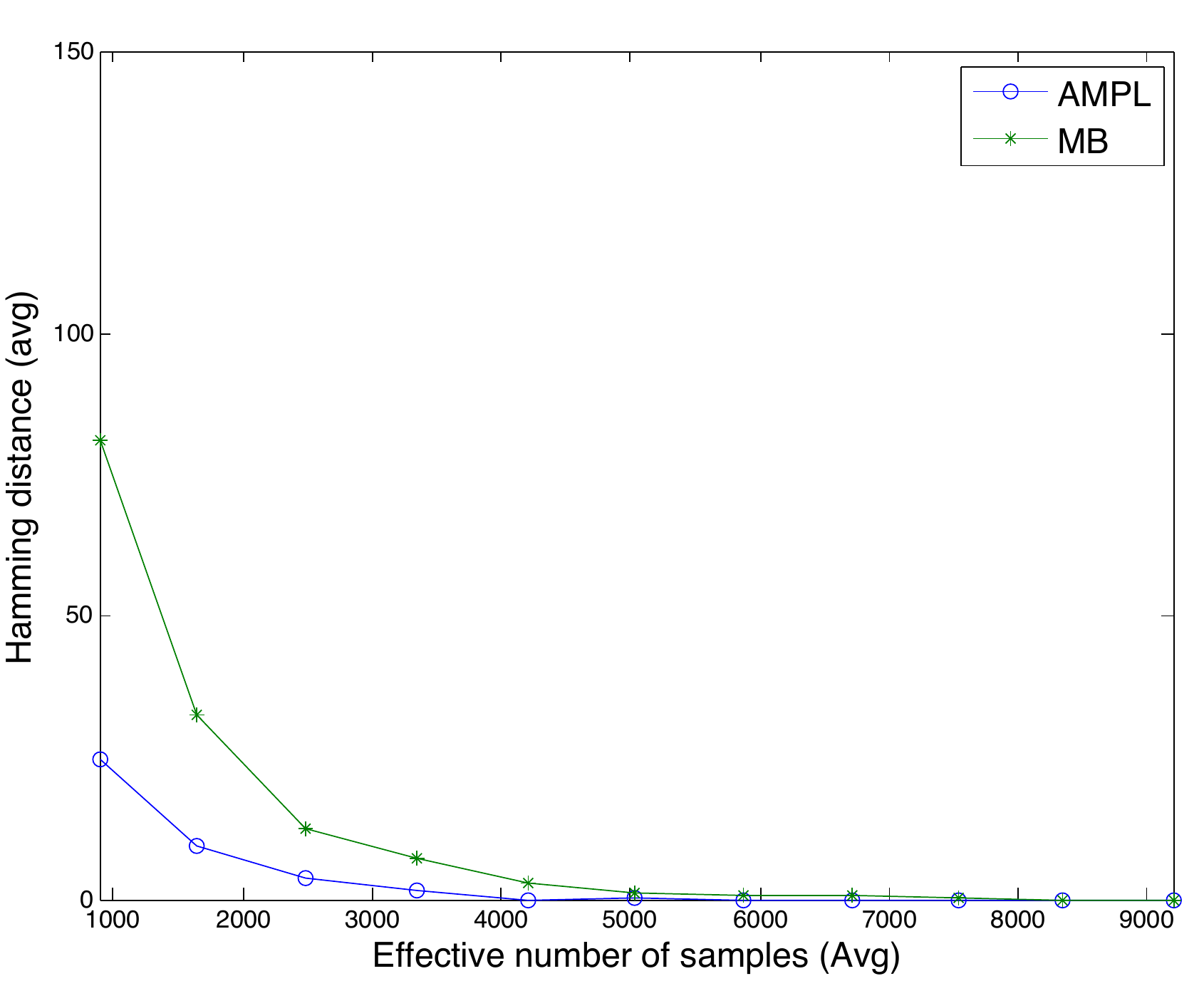}
\caption*{{Power Law}}
\label{fig:singleClique}
\end{minipage}
\caption{Plot of Effective number of Samples (total number of scalars/$p$) -v- Hamming Error}
\label{plot:hamming}
\end{figure*}

We will now describe some preliminary experimental results. In particular, we will focus on the computationally efficient AMPL algorithm (Algorithm~\ref{alg:LassoAlg}) and compare it to its natural passive counterpart, the neighborhood regression algorithm from \cite{meinshausen2006high} (henceforth, MB). We implemented each algorithm using the Glmnet package \cite{glmnet}, where we tuned the parameters so that each algorithm achieved the best model selection performance against the true model. 

We evaluated these algorithms on the following graphs: \\
(a) {\bf Single Clique : } a graph on $p=60$ vertices composed of a clique on $12$ vertices and a chain graph connecting the rest. Observe that here $d_{\rm max} = 11$ and $\bar{d}_{\rm max} = 3.8$. \\
(b) {\bf Multiple Cliques: } a graph  on $p=100$ vertices with 4 (disconnected) cliques of sizes $5, 8, 10, 11$ respectively; the rest of the graph is again connected with a chain. Here, $d_{\rm max} = 10$ and $\bar{d}_{\rm max} = 4.08$ \\
(c) {\bf Power Law: }a (random) power law graph on $p = 60$ vertices generated according to the Barab\'{a}si-Albert model with a randomly generate sparse seed graph on $5$ vertices. These graphs are such that the number of vertices of degree $x$ in the graph behaves like $a*x^{-b}$ for constants $a,b > 0$.  See \cite{albert2002statistical} for more on these graphs. For a typical graph we generated, we onserved that $d_{\rm max} = 13$ and $\bar{d}_{\rm max} = 3.68$. 

\begin{table}[]
\centering
\caption{Comparison of ESC (10 trials)}
\label{table:results}
\begin{tabular}{|l|cc|ll|}
\hline
\multicolumn{1}{|c|}{Graph} & \multicolumn{2}{c|}{ESC (0.9)}               & \multicolumn{2}{c|}{ESC(1)} \\ \cline{2-5} 
                            & AMPL & MB ( \cite{meinshausen2006high} )  & AMPL        & MB (\cite{meinshausen2006high} )          \\ \hline
Single Clique               & 1202.1       & 3361.8                          & 1202       & 3361.9       \\
Mult. Cliques            & 1154.3      & 2943.8                          & 2649.5       & 6216.1       \\
Power Law                   & 1280.2       & 2300.4                          & 4212.8       & 8004.7       \\ \hline
\end{tabular}
\end{table}

In Table~\ref{table:results}, we report the comparison between AMPL and MB. 
To run AMPL, we choose the constant $c$ (cf. Section~\ref{sec:lassoAlgo}) and allow the algorithm to run until it reports all vertices are settled. We then record $(i)$ the total number of scalar samples consumed by the algorithm divided by $p$; we call this the effective sample complexity, and $(ii)$ its model selection performance in terms of the hamming distance between the estimated and the true adjacency matrices. We then ran MB for various sample sizes and recorded its hamming error performance. In Table~\ref{table:results}, we report two numbers for each algorithm and each graph: (a) ESC(0.9): the average (over 10 trials) number of effective samples required to get atleast 90\% of the edges, and (b) ESC(1): the average number of  effective samples required to get all the edges right. Notice that ESC(1)$\times p$ is the total number of scalar samples consumed. 

As we can see from Table~\ref{table:results}, the AMPL algorithm clearly outperforms the neighborhood regression algorithm of Meinshausen and Buhlmann in the situations where the degree distribution of the true graphs is non-uniform. 

We also compare the ``average'' hamming error performance of these two algorithms in Figure~\ref{plot:hamming}. Since the total sample complexity of the AMPL algorithm is a random number, we adopt the following protocol to generate these plots: We first choose a value for the constant $c$ (cf. Section~\ref{sec:lassoAlgo}) and allow AMPL  to run to completion. We then compute the effective number of samples  by dividing the total number of scalars by the size of the graph ($p$). These many independent samples are then fed to MB and its hamming error performance is recorded. For each value of $c$, we repeat this 10 times and record the average effective number of samples and the average hamming error for both AMPL and MB. These results are shown in Figure~\ref{plot:hamming}.

\section{Discussion}
In this paper we introduced a generic framework for active graphical model selection. We then showed two instances of this algorithm: AdPaCT (Algorithm~\ref{alg:exhaustiveAlg}) and AMPL (Algorithm~\ref{alg:LassoAlg}) both of which, under certain standard assumptions, achieve a total sample complexity of $\mathcal{O}(\bar{d}_{\rm max} p \log p)$ as opposed to the total sample complexity of $\mathcal{O}(d_{\rm max} p \log p)$ of their passive counter parts. By definition, $\bar{d}_{\rm max}$ is no larger than $d_{\rm max}$. But in many situations, for instance when the graph has a highly non-uniform degree distribution, the difference between these quantities can be significant (see Section~\ref{sec:dbarmax} for more). Therefore, the algorithms presented in this paper provably perform better than their passive counterparts. Furthermore, in Section~\ref{sec:simulations}, we illustrate our theoretical findings via simulations. 

As is clear from Theorem~\ref{thm:metaAlg}, the framework we propose can also be extended beyond the Gaussian setting provided one can produce subroutines {\tt\bf nbdSelect} and {\tt\bf nbdVerify} that satisfy conditions (C1) and (C2). For instance, if one were to consider the reconstruction of Ising models, then $\ell_1$ regularized logistic regression can be used for {\tt\bf nbdSelect} as suggested by \cite{ravikumar2010high} and mutual information based conditional indpendence tests can be used for {\tt\bf nbdVerify}. We leave a careful analysis of this case to future work. Activizing non-neighborhood based approaches for graphical model selection like \cite{ravikumar2011high, anandkumar2012gaussian} is an interesting avenue for future work as well. Finally, it will also be of interest to show tight necessary conditions for such active marginalization based graph learning algorithms. 
\pagebreak
\bibliographystyle{ieeetr}
\bibliography{gmbib}

\begin{thebibliography}{10}

\bibitem{marvcenko1967distribution}
V.~A. Mar{\v{c}}enko and L.~A. Pastur, ``Distribution of eigenvalues for some
  sets of random matrices,'' {\em Sbornik: Mathematics}, vol.~1, no.~4,
  pp.~457--483, 1967.

\bibitem{johnstone2001distribution}
I.~M. Johnstone, ``On the distribution of the largest eigenvalue in principal
  components analysis,'' {\em Annals of statistics}, pp.~295--327, 2001.

\bibitem{ravikumar2010high}
P.~Ravikumar, M.~J. Wainwright, and J.~D. Lafferty, ``High-dimensional ising
  model selection using ℓ1-regularized logistic regression,'' {\em The Annals
  of Statistics}, vol.~38, no.~3, pp.~1287--1319, 2010.

\bibitem{ravikumar2011high}
P.~Ravikumar, M.~J. Wainwright, G.~Raskutti, and B.~Yu, ``High-dimensional
  covariance estimation by minimizing ℓ1-penalized log-determinant
  divergence,'' {\em Electronic Journal of Statistics}, vol.~5, pp.~935--980,
  2011.

\bibitem{anandkumar2012ising}
A.~Anandkumar, V.~Y. Tan, F.~Huang, A.~S. Willsky, {\em et~al.},
  ``High-dimensional structure estimation in ising models: Local separation
  criterion,'' {\em The Annals of Statistics}, vol.~40, no.~3, pp.~1346--1375,
  2012.

\bibitem{anandkumar2012gaussian}
A.~Anandkumar, V.~Y. Tan, F.~Huang, and A.~S. Willsky, ``High-dimensional
  gaussian graphical model selection: Walk summability and local separation
  criterion,'' {\em The Journal of Machine Learning Research}, vol.~13, no.~1,
  pp.~2293--2337, 2012.

\bibitem{keshri2013shotgun}
S.~Keshri, E.~Pnevmatikakis, A.~Pakman, B.~Shababo, and L.~Paninski, ``A
  shotgun sampling solution for the common input problem in neural connectivity
  inference,'' {\em arXiv preprint arXiv:1309.3724}, 2013.

\bibitem{turaga2013inferring}
S.~Turaga, L.~Buesing, A.~M. Packer, H.~Dalgleish, N.~Pettit, M.~Hausser, and
  J.~Macke, ``Inferring neural population dynamics from multiple partial
  recordings of the same neural circuit,'' in {\em Advances in Neural
  Information Processing Systems}, pp.~539--547, 2013.

\bibitem{bishop2014deterministic}
W.~E. Bishop and M.~Y. Byron, ``Deterministic symmetric positive semidefinite
  matrix completion,'' in {\em Advances in Neural Information Processing
  Systems}, pp.~2762--2770, 2014.

\bibitem{sachs2005causal}
K.~Sachs, O.~Perez, D.~Pe'er, D.~A. Lauffenburger, and G.~P. Nolan, ``Causal
  protein-signaling networks derived from multiparameter single-cell data,''
  {\em Science}, vol.~308, no.~5721, pp.~523--529, 2005.

\bibitem{meinshausen2006high}
N.~Meinshausen and P.~B{\"u}hlmann, ``High-dimensional graphs and variable
  selection with the lasso,'' {\em The Annals of Statistics}, pp.~1436--1462,
  2006.

\bibitem{balcan2006agnostic}
M.-F. Balcan, A.~Beygelzimer, and J.~Langford, ``Agnostic active learning,'' in
  {\em Proceedings of the 23rd international conference on Machine learning},
  pp.~65--72, ACM, 2006.

\bibitem{dasgupta2007general}
S.~Dasgupta, C.~Monteleoni, and D.~J. Hsu, ``A general agnostic active learning
  algorithm,'' in {\em Advances in neural information processing systems},
  pp.~353--360, 2007.

\bibitem{castro2008minimax}
R.~M. Castro and R.~D. Nowak, ``Minimax bounds for active learning,'' {\em
  Information Theory, IEEE Transactions on}, vol.~54, no.~5, pp.~2339--2353,
  2008.

\bibitem{beygelzimer2010agnostic}
A.~Beygelzimer, J.~Langford, Z.~Tong, and D.~J. Hsu, ``Agnostic active learning
  without constraints,'' in {\em Advances in Neural Information Processing
  Systems}, pp.~199--207, 2010.

\bibitem{hanneke2014theory}
S.~Hanneke, ``Theory of disagreement-based active learning,'' {\em Foundations
  and Trends in Machine Learning}, vol.~7, no.~2-3, pp.~131--309, 2014.

\bibitem{arias2013fundamental}
E.~Arias-Castro, E.~J. Candes, M.~Davenport, {\em et~al.}, ``On the fundamental
  limits of adaptive sensing,'' {\em Information Theory, IEEE Transactions on},
  vol.~59, no.~1, pp.~472--481, 2013.

\bibitem{malloy2014near}
M.~L. Malloy and R.~D. Nowak, ``Near-optimal adaptive compressed sensing,''
  {\em Information Theory, IEEE Transactions on}, vol.~60, no.~7,
  pp.~4001--4012, 2014.

\bibitem{haupt2009distilled}
J.~Haupt, R.~Castro, and R.~Nowak, ``Distilled sensing: Selective sampling for
  sparse signal recovery,'' in {\em International Conference on Artificial
  Intelligence and Statistics}, pp.~216--223, 2009.

\bibitem{eriksson2011active}
B.~Eriksson, G.~Dasarathy, A.~Singh, and R.~Nowak, ``Active clustering: Robust
  and efficient hierarchical clustering using adaptively selected
  similarities,'' {\em arXiv preprint arXiv:1102.3887}, 2011.

\bibitem{krishnamurthy2012efficient}
A.~Krishnamurthy, S.~Balakrishnan, M.~Xu, and A.~Singh, ``Efficient active
  algorithms for hierarchical clustering,'' {\em arXiv preprint
  arXiv:1206.4672}, 2012.

\bibitem{voevodski2012active}
K.~Voevodski, M.-F. Balcan, H.~R{\"o}glin, S.-H. Teng, and Y.~Xia, ``Active
  clustering of biological sequences,'' {\em The Journal of Machine Learning
  Research}, vol.~13, no.~1, pp.~203--225, 2012.

\bibitem{ailon2013breaking}
N.~Ailon, Y.~Chen, and X.~Huan, ``Breaking the small cluster barrier of graph
  clustering,'' {\em arXiv preprint arXiv:1302.4549}, 2013.

\bibitem{vats2014active}
D.~Vats, R.~D. Nowak, and R.~G. Baraniuk, ``Active learning for undirected
  graphical model selection,'' in {\em Proceedings of the 17th International
  Conference on Artificial Intelligence and Statistics (AISTATS)}, JMLR : W\&P
  volume 33, 2014.

\bibitem{tong2000active}
S.~Tong and D.~Koller, ``Active learning for parameter estimation in bayesian
  networks,''

\bibitem{tong2001active}
S.~Tong and D.~Koller, ``Active learning for structure in bayesian networks,''

\bibitem{murphy2001active}
K.~P. Murphy, ``Active learning of causal bayes net structure,'' {\em Technical
  Report, UC Berkeley}, 2001.

\bibitem{yuan2007model}
M.~Yuan and Y.~Lin, ``Model selection and estimation in the gaussian graphical
  model,'' {\em Biometrika}, vol.~94, no.~1, pp.~19--35, 2007.

\bibitem{yuan2010high}
M.~Yuan, ``High dimensional inverse covariance matrix estimation via linear
  programming,'' {\em The Journal of Machine Learning Research}, vol.~11,
  pp.~2261--2286, 2010.

\bibitem{xue2012nonconcave}
L.~Xue, H.~Zou, T.~Cai, {\em et~al.}, ``Nonconcave penalized composite
  conditional likelihood estimation of sparse ising models,'' {\em The Annals
  of Statistics}, vol.~40, no.~3, pp.~1403--1429, 2012.

\bibitem{kalisch2007estimating}
M.~Kalisch and P.~B{\"u}hlmann, ``Estimating high-dimensional directed acyclic
  graphs with the pc-algorithm,'' {\em The Journal of Machine Learning
  Research}, vol.~8, pp.~613--636, 2007.

\bibitem{lauritzen1996graphical}
S.~L. Lauritzen, {\em Graphical models}.
\newblock Oxford University Press, 1996.

\bibitem{wainwright2009sharp}
M.~J. Wainwright, ``Sharp thresholds for high-dimensional and noisy sparsity
  recovery using-constrained quadratic programming (lasso),'' {\em Information
  Theory, IEEE Transactions on}, vol.~55, no.~5, pp.~2183--2202, 2009.

\bibitem{spirtes2000causation}
P.~Spirtes, C.~N. Glymour, and R.~Scheines, {\em Causation, prediction, and
  search}, vol.~81.
\newblock MIT press, 2000.

\bibitem{wang2010information}
W.~Wang, M.~J. Wainwright, and K.~Ramchandran, ``Information-theoretic bounds
  on model selection for gaussian markov random fields,'' in {\em Information
  Theory Proceedings (ISIT), 2010 IEEE International Symposium on},
  pp.~1373--1377, IEEE, 2010.

\bibitem{chan2007approximation}
T.-H.~H. Chan, {\em Approximation algorithms for bounded dimensional metric
  spaces}.
\newblock ProQuest, 2007.

\bibitem{slivkins2007distance}
A.~Slivkins, ``Distance estimation and object location via rings of
  neighbors,'' {\em Distributed Computing}, vol.~19, no.~4, pp.~313--333, 2007.

\bibitem{pearl2000causality}
J.~Pearl, {\em Causality: models, reasoning and inference}, vol.~29.
\newblock Cambridge Univ Press, 2000.

\bibitem{tibshirani1996regression}
R.~Tibshirani, ``Regression shrinkage and selection via the lasso,'' {\em
  Journal of the Royal Statistical Society. Series B (Methodological)},
  pp.~267--288, 1996.

\bibitem{van2009conditions}
S.~A. Van De~Geer, P.~B{\"u}hlmann, {\em et~al.}, ``On the conditions used to
  prove oracle results for the lasso,'' {\em Electronic Journal of Statistics},
  vol.~3, pp.~1360--1392, 2009.

\bibitem{glmnet}
J.~Qian, T.~Hastie, J.~Friedman, R.~Tibshirani, and N.~Simon, ``Glmnet for
  matlab (2013),''

\bibitem{albert2002statistical}
R.~Albert and A.-L. Barab{\'a}si, ``Statistical mechanics of complex
  networks,'' {\em Reviews of modern physics}, vol.~74, no.~1, p.~47, 2002.

\bibitem{gubner2006probability}
J.~A. Gubner, {\em Probability and random processes for electrical and computer
  engineers}.
\newblock Cambridge University Press, 2006.

\bibitem{golub2012matrix}
G.~H. Golub and C.~F. Van~Loan, {\em Matrix computations}, vol.~3.
\newblock JHU Press, 2012.

\bibitem{horn2012matrix}
R.~A. Horn and C.~R. Johnson, {\em Matrix analysis}.
\newblock Cambridge university press, 2012.

\bibitem{davidson2001local}
K.~R. Davidson and S.~J. Szarek, ``Local operator theory, random matrices and
  banach spaces,'' {\em Handbook of the geometry of Banach spaces}, vol.~1,
  pp.~317--366, 2001.

\end{thebibliography}

\appendices
\section{Proof of Theorem~\ref{thm:metaAlg}}
\label{sec:proofOfTheorem1}

We will restate the theorem here for convenience. 
\newtheorem*{thm1}{\bf Theorem 1}
\begin{thm1}
%\label{thm:metaAlg}
Fix $\delta \in (0,1)$. For each $\ell \leq d_{\rm max}$, assume that the subroutines {\tt\bf nbdSelect} and {\tt\bf nbdVerify} satisfy: 
\begin{enumerate}
\item[(C1)] For any vertex $i\in [p]$ and subset $F\subseteq [p]$ that are such that $\left| N(i) \right| = d_i \leq \ell$ and $\overline{N}(i) \subseteq F$, the following holds. Given $g(\ell)$ samples from $X_F$, {\tt\bf nbdSelect}$(i, \ell, \left\{ X_F^{(j)} \right\}_{j\in S_1})$ returns the true neighborhood of $i$ with probability greater than $1-\delta/2p d_{\rm max}$. 
\item[(C2)] For any vertex $i\in [p]$ and subsets $F, H\subseteq [p]$ that are such that $\left| N(i) \right| = d_i \leq \ell$, $\overline{N}(i) \subseteq F$, and $H\subseteq F$, the following holds. Given $h(\left| H \right|)$ samples from $X_F$, {\tt\bf nbdVerify} $(i,H,\left\{ X_F^{(j)} \right\})$ returns {\bf true} if and only if $N(i)\subseteq H$ with probability greater than $1-\delta/2p d_{\rm max}$. 
\end{enumerate}
Then, with probability no less than $1-\delta$, Algorithm~\ref{alg:meta} returns the correct graph. Furthermore, it suffices if $B\geq \sum_{i\in [p]}\sum_{0\leq k \leq \lceil \log_2d_{\rm max}^i \rceil} g(2^k)  + h(2^k))$. That is, Algorithm~\ref{alg:meta} has a total sample complexity of $\sum_{i\in [p]}\sum_{0\leq k \leq \lceil \log_2d_{\rm max}^i \rceil} g(2^k)  + h(2^k)$ at confidence level $1-\delta$. 
\end{thm1}

\begin{proof}
To prove this theorem, we will use a simple argument that can be thought of as a proof by probabilistic induction. Towards this end, we will let $\mathcal{E}_k$ be the event that Algorithm~\ref{alg:meta} succeeds at iteration number $k$. Notice that $k$ takes values in the set $\left\{ 1,2,\ldots, \lfloor\log_2 (2p)\rfloor \right\}$ since the algorithm terminates when the (doubling) counter satisfies $\ell = 2^{k-1} \geq 2p$. We can characterize the event $\mathcal{E}_k$ as follows: 
\begin{itemize}
\item For each $i\in [p]\setminus \mbox{\sc nbdFound}$, if $d_i = \left| N(i) \right|\leq 2^{k-1}$, then $\widehat{N}(i)$, the output of \\{\tt nbdSelect}$( i, \ell, \{X_{[p]\setminus\mbox{\sc settled}}^{(j)}\}_{j\in S_1})$ is exactly $N(i)$, and {\tt nbdVerify}$( i,\widehat{N}(i), \left\{ X_{[p]\setminus \mbox{\tiny{\sc settled}}} \right\}_{j\in S_2})$ outputs {\bf true}. 
\item For each $i\in [p]\setminus \mbox{\sc nbdFound}$, if $d_i > 2^{k-1}$, then {\tt nbdVerify}$( i,\widehat{N}(i), \left\{ X_{[p]\setminus \mbox{\tiny{\sc settled}}} \right\}_{j\in S_2})$ outputs {\bf false}. 
\end{itemize} 

We begin our proof by bounding the probability of error from above as follows. 
\begin{align}
\mathbb{P}\left[ \mbox{error} \right] &= \mathbb{P}\left[ \bigcup_{k = 1}^{\lfloor \log_2 (2p) \rfloor} \mathcal{E}_k^c \right]\nonumber\\
&\leq \sum_{k = 1}^{\lfloor \log_2 (2p) \rfloor} \mathbb{P}\left[ \mathcal{E}_{k}^c\middle | \mathcal{E}_1,\ldots, \mathcal{E}_{k - 1} \right], \label{eq:theorem1UnionBound}
\end{align}

where we have used the convention that $\mathcal{E}_1,\ldots, \mathcal{E}_{k - 1} = \emptyset$ for $k = 1$. In what follows, fix an arbitrary $k\in [\lfloor \log_2(2p) \rfloor]$ and set $\ell = 2^{k-1}$; we will bound the probability $\mathbb{P}\left[ \mathcal{E}_k^c \middle| \mathcal{E}_1, \ldots, \mathcal{E}_{k-1} \right]$.

First, observe that conditioned on $\mathcal{E}_1,\ldots, \mathcal{E}_{k - 1}$, the following is true of the (evolving) sets {\sc nbdFound} and {\sc settled}. If a vertex $i\in [p]$ is such that $d_i = \left| N(i) \right| \leq \ell/2$, then {\sc nbdFound} contains $i$, and similarly if every $j\in N(i)$ is such that $d_j \leq \ell/2$, then {\sc settled} contains $i$. Next, we observe that since we are conditioning on $\mathcal{E}_1,\ldots, \mathcal{E}_{k - 1}$, the following statements hold provided $\ell \leq d_{\rm max}$: 
\begin{itemize}
\item If $i$ is such that $d_i \leq \ell$, then $\overline{N}(i) \cap \mbox{\sc settled} = \emptyset$, since each $j\in N(i)$ has at least one neighbor (viz., $i$) has not been enrolled in {\sc nbdFound}. Therefore, by (C1), $\widehat{N}(i)$, the output of \\{\tt nbdSelect}$(i, \ell, \left\{ X^{(j)}_{[p]\setminus \mbox{\sc settled}} \right\}_{j \in S_1})$ is exactly  $N(i)$ with probability at least $ 1- \delta/2p d_{\rm max}$. 
\item For such an $i$ (i.e., with  $d_i \leq \ell$), {\tt nbdSelect}$(i, \widehat{N}(i), \{ X^{(j)}_{[p]\setminus \mbox{\sc settled}}\}_{j \in S_2})$ returns {\bf true}. On the other hand, if $i$ is such that $d_i > \ell$,  the subroutine {\tt nbdVerify}$(i, \widehat{N}(i), \{ X^{(j)}_{[p]\setminus \mbox{\sc settled}} \}_{j \in S_2})$ returns {\bf false} since $\left| \widehat{N}(i) \right|\leq \ell$ by definition of the {\tt nbdSelect} function. Both these follow from (C2) and with probability at least $1 - \delta/2p d_{\rm max}$.
\end{itemize}
Both these observations together imply that for any $k$ such that $2^{k-1} \leq d_{\rm max}$, $\mathbb{P}\left[ \mathcal{E}_k^c\middle | \mathcal{E}_1, \ldots, \mathcal{E}_{k-1} \right] \leq p/d_{\rm  max}$; observe that we have (quite conservatively) bounded this event by using $p$ as an upper bound to the number of vertices whose degrees do not exceed $2^{k-1}$.  On the other hand, if $2^{k-1} > d_{\rm max}$, by observations made above, $\mbox{\sc nbdFound} = [p]$. Therefore, Algorithm~\ref{alg:meta} would have terminated before $k$ reached such a value and $\mathbb{P}\left[ \mathcal{E}_k^c\middle | \mathcal{E}_1, \ldots, \mathcal{E}_{k-1} \right] = 0$. Therefore, from \eqref{eq:theorem1UnionBound}, we have that the probability that the algorithm errs is no more than $\delta$, as required.  

Finally, observe that the above argument implies that with probability greater $1 - p d_{\rm max} \delta$, the following is true. Each vertex $i\in [p]$ is enrolled in {\sc nbdFound} no later than when the counter reaches $\ell = 2^{\lceil \log_2 d_i\rceil} \leq 2 d_i$. 
%%$\ell = d_i$, vertex $i$ is enrolled in  
Therefore, by the time  $\ell$ reaches $2d_{\rm max}^i$, every neighbor of $i$ has already been enrolled in {\sc nbdFound}, which of course implies that $i$ is enrolled in {\sc settled} and is no longer sampled from. Therefore, the total number of samples accumulated for vertex $i$ is given by $\sum_{k=0}^{\lceil \log_2d_{\rm max}^i \rceil} g(2^k)  + h(2^k)$. This implies that a budget $B\geq \sum_{i\in [p]}\sum_{0\leq k \leq \lceil \log_2d_{\rm max}^i \rceil} g(2^k)  + h(2^k)$ is sufficient. 
\end{proof}

\section{Proof of Theorem~\ref{thm:exhaustiveAlg}}
\label{sec:exhaustiveAlgProof}
We will restate Theorem~\ref{thm:exhaustiveAlg} here for convenience. 
\newtheorem*{thm2}{\bf Theorem 2}

\begin{thm2}
Fix $\delta\in (0,1)$ and suppose that assumptions (A1) and (A2) hold. Then, there exists a constant $c = c(m, M, \delta)$ such that if we set $g(\ell) = \lceil c \ell \log p \rceil$ and $\xi = m/2$, then with probability no less than $1  - \delta$, the following hold: 
\begin{enumerate}
\item The AdPaCT algorithm successfully recovers the graph $G$.
\item The computational complexity of the AdPaCT algorithm is no worse than $\mathcal{O}(p^{d_{\rm max} + 2})$
\end{enumerate}
This implies that the total sample complexity of the AdPaCT algorithm at confidence level $1 - \delta$ is bounded by $2c \bar{d}_{\rm max} p\log p$. 
\end{thm2}

\begin{proof}To prove Theorem~\ref{thm:exhaustiveAlg}, we will bound the probability that the conditions (C1) and (C2) are violated. 

First, fix an arbitrary $\ell \leq d_{\rm max}$, a vertex $i\in [p]$, and a subset $F\subseteq [p]$ such that $d_i \leq \ell$ and $\overline{N}(i)\subseteq F$. For ease of notation, we will let $n_\ell = g(\ell)$. The event that the {\bf nbdVerify} subroutine defined in Algorithm~\ref{alg:exhaustiveAlg} does not satisfy the condition (C1) is equivalent to saying that there is a set $S\subseteq F$ such that $\left| S \right|\leq \ell$ and $\max_{j\notin S\cup\{i\}}\left| \widehat{\rho}_{i,j\mid S} \right|\leq \xi$ and one of the following events hold: 
\begin{itemize}
\item $\max_{j\notin \overline{N}(i)}\left| \widehat{\rho}_{i,j\mid N(i)} \right| > \xi$
\item $\left| S \right|	 < d_i$
\item $\left| S \right| = d_i$ and $\max_{j\notin S\cup\{i\}}\left| \widehat{\rho}_{i,j\mid S} \right| \leq \max_{j\notin \overline{N}(i)}\left| \widehat{\rho}_{i,j\mid N(i)} \right|$.
\end{itemize}
Letting $\mathcal{S}_{i,j,\ell}$ denote the set of all sets of size at most $\ell$ that {\bf do not} separate\footnote{A set $S$ is said to separate a pair of vertices $i$ and $j$ in a graph if all the paths that connect $i$ and $j$ contain at least one vertex from $S$.} $i$ from $j$ in the graph $G$, observe that above events imply that one or both of the following conditions hold: (a) there exists a vertex $j\in [p]\setminus\{i\}$ and a subset $S\subseteq \mathcal{S}_{i,j,\ell}$ such that $\left| \widehat{\rho}_{i,j\mid S} \right|\leq \xi$, or (b) there is a vertex $j\in [p]\setminus \overline{N}(i)$: $\left| \widehat{\rho}_{i,j\mid N(i)} \right| > \xi$. Therefore, we will bound the probability that (C1) does not hold, an event we will dub $\mathcal{E}_1$, as follows 
\begin{align}
\mathbb{P}\left[ \mathcal{E}_1 \right] &\leq \sum_{\substack{j\in [p]\setminus \{i\}\\ S\in \mathcal{S}_{i,j,\ell}}}\mathbb{P}\left[ \left| \widehat{\rho}_{i,j\mid S} \right| \leq \xi \right] + \sum_{j\in [p]\setminus \overline{N}(i)} \mathbb{P}\left[ \left| \widehat{\rho}_{i,j\mid N(i)} \right| > \xi \right]. \label{eq:exhaustiveAlgoE1}
\end{align}

To proceed, let us consider an arbitrary term in the first sum. Since $S\in \mathcal{S}_{i,j,\ell}$, we know by the second part of assumption (A2) that $\left| \rho_{i,j \mid S} \right| > m$. Now, observe that $\left| \widehat{\rho}_{i,j\mid S} \right|\leq \xi$ and $\left| \rho_{i,j\mid S} \right| \geq m$ together imply that $\left| \rho_{i,j\mid S} \right| - \left| \widehat{\rho}_{i,j\mid S}\right| \geq m - \xi \Rightarrow \left| \rho_{i,j\mid S} - \widehat{\rho}_{i,j\mid S} \right| \geq m - \xi$, since $m > \xi$. Therefore, in this case, we have
\begin{align}
\mathbb{P}\left[ \left| \widehat{\rho}_{i,j\mid S} \right| \leq \xi  \right] &\leq \mathbb{P}\left[ \left| \widehat{\rho}_{i,j\mid S} - \rho_{i,j\mid S} \right| \geq m-\xi  \right]\\
&\stackrel{(a)}{\leq} C_1\left( n_\ell - 2 - \left| S \right| \right)\exp\left\{ - \left( n_\ell - 4 - \left| S \right| \right)\log\left( \frac{4+\left( m-\xi \right)^2}{ 4 - (m-\xi)^2 } \right) \right\}\\
&\leq  C_1n_\ell\exp\left\{ - \left( n_\ell- 4 - \ell \right)\log\left( \frac{16+m^2}{16 - m^2 } \right) \right\} \label{eq:conditionallyDependent}, 
\end{align}
where $(a)$ follows from Lemma~\ref{lemma:partialCorrelationConcentration} in Appendix~\ref{sec:lemmata} and as in the lemma, the constant $C_1$ only depends on $M$. The last step follows from the condition that $\left| S \right|\leq \ell$ and $\xi = m/2$. 

Next, we will consider an arbitrary term in the second summation of \eqref{eq:exhaustiveAlgoE1}. Since $j\notin \overline{N}(i)$, we know by assumption (A1), $\rho_{i,j\mid N(i)} = 0$. Therefore, 
\begin{align}
\mathbb{P}\left[ \left| \widehat{\rho}_{i,j\mid S} \right| \geq \xi  \right] &= \mathbb{P}\left[ \left| \widehat{\rho}_{i,j\mid S} - \rho_{i,j \mid S}\right| \geq \xi  \right]\\
&\stackrel{(a)}{\leq} C_1 \left( n_\ell - 2 - \left| S \right| \right) \exp\left\{  - \left( n_\ell  - 4 - \left| S \right| \right)\log\left( \frac{4 + \xi^2}{4 - \xi^2} \right)  \right\}\\
& \leq C_1 n_\ell \exp\left\{  - \left( n_\ell  - 4 - \ell \right)\log\left( \frac{16 + m^2}{16 - m^2} \right)  \right\}, \label{eq:condionallyIndependent}
\end{align}
where, again, $(a)$ follows from Lemma~\ref{lemma:partialCorrelationConcentration} in Appendix~\ref{sec:lemmata} and the last step follows after observing that $\left| S \right|\leq\ell$ and $\xi = m/2$. So, from \eqref{eq:exhaustiveAlgoE1}, we have the following upper bound on the probability that (C1) is violated: 
\begin{align}
\mathbb{P}\left[ \mbox{(C1) is violated for a fixed $i$} \right] &\leq \sum_{\substack{j\in [p]\setminus \{i\}\\ S\in \mathcal{S}_{i,j,\ell}}}\mathbb{P}\left[ \left| \widehat{\rho}_{i,j\mid S} \right| \leq \xi \right] + \sum_{j\in [p]\setminus \overline{N}(i)} \mathbb{P}\left[ \left| \widehat{\rho}_{i,j\mid N(i)} \right| > \xi \right]\\
&\leq 2C_1 p^{\ell + 1} n_\ell \exp\left\{  - \left( n_\ell  - 4 - \ell \right)\log\left( \frac{16 + m^2}{16 - m^2} \right)  \right\},
\end{align}
where the last step follows from observing that there are no more than $p^{\ell + 1}$ terms in the first sum and $p$ terms in the second sum. As observed in Section~\ref{sec:exhaustiveAlg}, (C2) is satisfied by default for Algorithm~\ref{alg:exhaustiveAlg} since the verification is implicitly performed by the exhaustive searching of the {\bf ndbSelect} subroutine. Therefore, if the number of samples $n_\ell = g(\ell)$ satisfies
\begin{equation}
n_\ell = \left\lceil\ell + 5 + \frac{1}{\log\left( \frac{16 + m^2}{16 - m^2} \right)} \left((\ell + 5)\log (2C_1 p) + \log(2/\delta)\right)	\right\rceil,
\end{equation}
then the following implications hold
\begin{align}
2C_1 p^{\ell + 1} n_\ell \exp &\left\{  - \left( n_\ell  - 4 - \ell \right)\log\left( \frac{16 + m^2}{16 - m^2} \right)  \right\}\nonumber\\
 &\leq 2C_1 p^{\ell+1} \left[ \ell + 6 + \frac{1}{\log\left( \frac{16 + m^2}{16 - m^2} \right)} \left((\ell + 5)\log (2C_1 p) + \log(2/\delta)\right)	 \right]\nonumber\\
&\qquad\qquad\qquad\exp\left\{ -\left( \ell + 5 \right)\log(2C_1 p) \right\}\times \frac{\delta}{2}\\
&\leq (2C_1)^{-\ell - 4} p^{-4} \left[p+6 + \frac{1}{\log\left( \frac{16 + m^2}{16 - m^2} \right)} \left((p + 5)\log (2C_1 p) + \log(2/\delta)\right)\right] \\
&\leq \frac{\delta}{2}p^{-2}, \; \mbox{for $p$ large enough}.
\end{align}

Therefore, from Theorem~\ref{thm:metaAlg}, we have that the \adpact{} algorithm succeeds with probability exceeding $ 1-pd_{\rm max} \delta p^{-2}\geq 1 - \delta$ (since $d_{\rm max}\leq p$), as required. The above calculation also demonstrates that (provided there is a constant $c''>0$ such that $\delta < p^{-c''}$, which is a mild requirement) there exists a constant $c' >0$ such that we might choose $g(\ell) = \lceil c' \ell \log (p)/2 \rceil$. Finally, using Theorem~\ref{thm:metaAlg}, we observe that the number of samples accumulated for vertex $i$ is no more than $\sum_{k \leq \lceil \log_2d_{\rm max}^i \rceil} g(2^k)  + h(2^k) \leq 2 c d^i_{\rm max} \log p/2$ (where $c$ accounts for the integer effects ).  Therefore, for the \adpact{} algorithm to succeed, it suffices to pick a budget that satisfies $B\geq c \bar{d}_{\rm max} p \log p $. Finally, observe that the computational complexity statement follows from the size of the subsets that are being searched over. This concludes the proof. \end{proof}

\section{Proof of Theorem~\ref{thm:lassoAlg}}
\label{sec:LassoAlgProof}
In this section, we will prove Theorem~\ref{thm:lassoAlg}, which we again restate here for convenience. 
\newtheorem*{thm3}{\bf Theorem 3}
\begin{thm3}
Fix $\delta >0$. Suppose that assumptions (A1)-(A4) hold. There exists constants $C_1, C_2, C_3$ which depend on $\Sigma, m, \delta$ such that if we set $c = C_1$ (i.e., $g(\ell) = c \ell \log p$), $\xi = m/2$, $\lambda_{\ell} = \sqrt{\frac{2 C_2 \left\| \Sigma \right\|_\infty}{C_1 \gamma^2}}$, and budget  $B= 2c\bar{d}_{\rm max} p\log p$, then with probability at least $1 - \delta$, the following hold
\begin{enumerate}
\item the AMPL algorithm successfully recovers the graph $G$, 
\item The computational complexity is bounded from above by $d_{\rm max}p \mathfrak{C}$, where $\mathfrak{C}$ is the computational cost of solving a single instance of Lasso,
\end{enumerate}
provided $m \geq \left( \frac{C_{\rm min}}{C_{\rm max}} + \frac{C_{\rm max}}{C_{\rm min}} + 2 \right)\times \frac{1}{4 \min_i \left| \Sigma_{ii} \right|} \left[ C_3 \sqrt{\frac{2C_1 \left\| \Sigma \right\|_\infty}{C_2 \gamma^2}}\max_{i} \vertiii{\left(\widetilde{\Sigma}^i_{N(i),N(i)}\right)^{-1/2}}^2_\infty + 20 \sqrt{\frac{\left\| \Sigma \right\|_\infty}{C_{\rm min} C_2}}\right]$.
 \end{thm3}

\begin{proof}As in the proof of Theorem~\ref{thm:exhaustiveAlg}, we will prove Theorem~\ref{thm:lassoAlg} by showing that the subroutines {\tt nbdSelect} and {\tt nbdVerify} from Algorithm~\ref{alg:LassoAlg} (AMPL) satisfy the conditions (C1) and (C2) of Theorem~\ref{thm:metaAlg}. Along the way, we will also identify the functions $g()$ and $h()$ which will suggest a choice for the budget.

\begin{lemma}
\label{lemma:AMPLC1Holds}
Fix an arbitrary $\ell \leq d_{\rm max}$. Let $i\in [p]$ and $F\subseteq [p]$ be such that $d_i \leq \ell$ and $\overline{N}(i)\subseteq F$. There exist constants $C_1, C_2, C_3>0$ such that the {\bf ndbSelect} subroutine returns $N(i)$ with probability greater than $1 - \delta/pd_{\rm max}$, provided the following hold 
\begin{enumerate}
\item $g(\ell) = C_1 \ell \log p$
\item $\lambda_\ell = \sqrt{\frac{2C_2 \left\| \Sigma \right\|_\infty}{C_1 \gamma^2}}$
\item $m \geq \left( \frac{C_{\rm min}}{C_{\rm max}} + \frac{C_{\rm max}}{C_{\rm min}} + 2 \right)\times \frac{1}{4 \min_i \left| \Sigma_{ii} \right|} \left[ C_3 \sqrt{\frac{2C_2 \left\| \Sigma \right\|_\infty}{C_1 \gamma^2}}\max_{i} \vertiii{\left(\widetilde{\Sigma}^i_{N(i),N(i)}\right)^{-1/2}}^2_\infty + 20 \sqrt{\frac{\left\| \Sigma \right\|_\infty}{C_{\rm min} C_2}}\right]$
\end{enumerate} 	
\end{lemma}

%\noindent\underline{{\bf nbdSelect} satisfies condition (C1) of Theorem~1}\todo{make as lemma}\\
\begin{proof}
%	Fix an arbitrary $\ell \leq d_{\rm max}$. Let $i\in [p]$ and $F\subseteq [p]$ be such that $d_i \leq \ell$ and $\overline{N}(i)\subseteq F$. 
The subroutine {\bf ndbSelect} receives as input $g(\ell)$ (which we will denote as $n_\ell$ for the rest of this proof) samples from the random variables $X_F = \left\{ X_i: i\in F \right\}$. Notice that $X_F$ is distributed according to $\mathcal{N}(0, \bar{\Sigma})$, where we write $\bar{\Sigma}$ to denote $\Sigma(F,F)$, the $(F,F)$ submatrix of $\Sigma$. 

First, we will begin by providing justification for the fact that the Lasso \cite{tibshirani1996regression} can be used for selecting the neighborhood of $i$ in our setting. The seminal work of \cite{meinshausen2006high} first recognized that the Lasso can be used for neighborhood selection in Gaussian graphical models. We will adapt this insight, while accounting for the sequential marginalization of our active algorithm. Towards this end, let $X_i$ denote the random variable corresponding to vertex $i$  and let $X_G$ denote the random vector corresponding to the vertices $G \triangleq F\setminus \left\{ i \right\}$. As noted above, $X_F \sim \mathcal{N}(0, \bar{\Sigma})$. Notice that the corresponding precision matrix, $\bar{K} \triangleq \bar{\Sigma}^{-1}$, is not equal to the original precision matrix $K$. We know that conditioned on $X_G$, $X_i$ behaves like a Gaussian random variable, and in particular, the conditional distribution takes the following form (see, e.g., \cite[Chapter 9]{gubner2006probability}): 
\begin{align}
p(X_i \mid X_G) = \mathcal{N}(\bar{\Sigma}_{iG}\bar{\Sigma}^{-1}_{GG}X_G, \bar{\Sigma}_{ii} - \bar{\Sigma}_{iG}\bar{\Sigma}_{GG}^{-1}\bar{\Sigma}_{Gi}).\label{eq:conditionalDistribution}
\end{align}
Now, we will let $y \in \mathbb{R}^{n_\ell}$ denote the vector of samples from $X_i$ that is received by {\bf ndbSelect} and we will let $\mathbf{X}\in \mathbb{R}^{n_\ell \times \left| G \right|}$ denote the matrix of corresponding samples from the random vector $X_G$. This notation will let us simplify the following presentation while allowing us to readily identify the components of our problem with the classic Lasso problem. With this notation, from \eqref{eq:conditionalDistribution}, we can write down the ``true model'' corresponding to our problem as 

\begin{equation}
y = \mathbf{X}\beta^\ast + w, \label{eq:AMPLProofModel}
\end{equation}
where $y, \mathbf{X}$ are as above, $\beta^\ast \triangleq  \bar{\Sigma}_{iG}\bar{\Sigma}^{-1}_{GG}\in \mathbb{R}^{\left| G \right|}$, and $w\in \mathbb{R}^{n_\ell}\stackrel{\rm iid}{\sim} \mathcal{N}\left(0, \bar{\Sigma}_{ii} - \bar{\Sigma}_{iG}\bar{\Sigma}_{GG}^{-1}\bar{\Sigma}_{Gi}\right)$. We will now show that recovering the support of $\beta^\ast$ suffices. 
\begin{claim}
\label{claim:betaStarIsRight}
For any $j\in F\setminus\left\{ i \right\}$, we have that $\beta^\ast_j = -\frac{K_{ij}}{K_{ii}}$. 
\end{claim}
\begin{proof}
To prove this, we will first begin by writing $\bar{\Sigma}$ and $\bar{K}$ in their partitioned form: 
\begin{align}
\bar{\Sigma} &= \left(\begin{matrix}
  \bar{\Sigma}_{ii}& \bar{\Sigma}_{iG}  \\
\bar{\Sigma}_{Gi}  & \bar{\Sigma}_{GG}
\end{matrix}\right)\\
\bar{K} &= \left( \begin{matrix}
 \bar{K}_{ii}& \bar{K}_{iG}  \\
\bar{K}_{Gi}  & \bar{K}_{GG}
\end{matrix}
 \right)
\end{align}	

Using a standard block matrix inversion (see, e.g., \cite{golub2012matrix}), we observe that 
\begin{align}
\left( \begin{matrix}
 \bar{K}_{ii}& \bar{K}_{iG}  \\
\bar{K}_{Gi}  & \bar{K}_{GG}
\end{matrix}
 \right)
= \bar{K} = \bar{\Sigma}^{-1} &= \left(\begin{matrix}
  \bar{\Sigma}_{ii}& \bar{\Sigma}_{iG}  \\
\bar{\Sigma}_{Gi}  & \bar{\Sigma}_{GG}
\end{matrix}\right)^{-1}\\
&= \left(\begin{matrix}
 \left( \bar{\Sigma}_{ii} - \bar{\Sigma}_{iG}\bar{\Sigma}_{GG}^{-1}\bar{\Sigma}_{Gi} \right)^{-1} & - \left(\bar{\Sigma}_{ii} - \bar{\Sigma}_{iG}\bar{\Sigma}_{GG}^{-1}\bar{\Sigma}_{Gi} \right)^{-1} \bar{\Sigma}_{iG}\bar{\Sigma}^{-1}_{GG} \\
\ast   &  \ast
\end{matrix}\right),
\end{align}
where the $\ast$ values maybe ignored for the present calculation. Comparing the two block matrices above, it becomes clear that $-\bar{K}_{ii}\bar{\Sigma}_{iG}\bar{\Sigma}^{-1}_{GG} = \bar{K}_{iG}$. Since we know from \eqref{eq:conditionalDistribution} that $\beta^\ast = \bar{\Sigma}_{iG}\bar{\Sigma}^{-1}_{GG}$, we have shown that $\beta^\ast_j = -\frac{\bar{K}_{ij}}{\bar{K}_{ii}}$. 

Recall that $\bar{\Sigma} = \Sigma(F,F)$. Therefore, arguing as above, we have the following relationship between the entries of $K$ and $\bar{K}$
\begin{align}
\bar{K} = K_{FF} - K_{FF^c}K_{F^cF^c}^{-1}K_{F^c F}.\label{eq:schur complement}
\end{align}
$\bar{K}$ is called the \emph{Schur complement} of the block $K_{F^c F^c}$ with respect to the matrix $K$. Now, by the hypothesis of the condition (C1), we know that $F^c \subseteq \bar{N}(i)^c$. This implies that $K_{iF^c} = 0$ identically, and in particular, this means that $\bar{K}_{ij} = K_{ij}$ for any $j\in F$. This concludes the proof. 
\end{proof}

Along with the fact that the non-zeros in the concentration matrix, $K$, correspond to graph edges, Claim~\ref{claim:betaStarIsRight} shows us that the support of $\beta^\ast$ from \eqref{eq:AMPLProofModel} gives us the neighborhood of $i$. 

Observe that the candidate neighborhood chosen by the {\bf ndbSelect} subroutine is the support of a vector $\widehat{\beta}\in \mathbb{R}^{\left| G \right|}$, where 
\begin{equation}
\widehat{\beta} \in \argmin_{\beta\in \mathbb{R}^{\left| G \right|}} \frac{1}{2 n_{\ell}}\left\| y - \mathbf{X}\beta \right\|_2^2 + \lambda_{\ell}\left\| \beta \right\|_1.\label{eq:ourLasso}
\end{equation}
Therefore, to conclude the proof of Lemma~\ref{lemma:AMPLC1Holds}, it suffices to show that $\widehat{\beta}$ and $\beta^\ast$ have the same support with high probability. For this, we will borrow the results of Theorem~\ref{thm:lassoSupportRecovery}, which is based on the seminal work of Wainwright \cite{wainwright2009sharp}. In order to apply this theorem in our setting, we need to ensure that the assumptions of the theorem are satisfied, and this is what the next claim will demonstrate.  
\begin{claim}
\label{claim:covMatrixIsOkay}
Let ${\breve{\Sigma}} \triangleq \Sigma(F\setminus\{i\},F\setminus\{i\})$  denote the covariance matrix corresponding to the rows of $\mathbf{X}$ and let $\beta^\ast_{\rm min} = \min_{j\in N(i)} |\beta^\ast_j|$. Then, the following hold
\begin{align}
&\vertiii{\breve{\Sigma}_{F\setminus \overline{N}(i), N(i)} \left( \breve{\Sigma}_{N(i),N(i)} \right)^{-1}}_\infty \leq 1 - \gamma\\
&0< C_{\rm min} \leq \Lambda_{\rm min}\left( \breve{\Sigma}_{N(i), N(i)} \right) \leq \Lambda_{\rm max}\left( \breve{\Sigma}_{N(i), N(i)} \right) \leq C_{\rm max} < \infty.\\
& \beta_{\rm min}^\ast \geq \frac{m \times \min_{i\in [p]}\Sigma_{ii}  }{\left( \frac{C_{\rm min}}{C_{\rm max}} + \frac{C_{\rm max}}{C_{\rm min}} + 2 \right)}
\end{align}
\end{claim}
\begin{proof}
First observe that the submatrices $\breve{\Sigma}_{N(i), N(i)}$ and $\widetilde{\Sigma}_{N(i), N(i)}$ are identical since, by assumption $N(i)\subseteq F\setminus \left\{ i \right\}$. Therefore, by the hypotheses of Theorem~\ref{thm:lassoAlg}, the second set of inequalities follow immediately. Similarly, we observe that since $F\setminus \overline{N}(i) \subseteq [p]\setminus\overline{N}(i)$, the hypothesis of Theorem~\ref{thm:lassoAlg} implies that 
\begin{align}
\vertiii{\breve{\Sigma}_{F\setminus \overline{N}(i), N(i)} \left( \breve{\Sigma}_{N(i),N(i)} \right)^{-1}}_\infty &= \max_{j\in F\setminus\overline{N}(i)} \sum_{r\in N(i)}\left| \sum_{t\in \overline{N}(i)}\breve{\Sigma}_{j,r} \left[\left( \breve{\Sigma}_{N(i),N(i)} \right)^{-1} \right]_{r,t}\right|\\
&= \max_{j\in F\setminus\overline{N}(i)} \sum_{r\in N(i)}\left| \sum_{t\in \overline{N}(i)}\breve{\Sigma}_{j,r} \left[\left( \widetilde{\Sigma}_{N(i),N(i)} \right)^{-1} \right]_{r,t}\right|\\
&\leq \max_{j\in [p]\setminus\overline{N}(i)} \sum_{r\in N(i)}\left| \sum_{t\in \overline{N}(i)}\breve{\Sigma}_{j,r} \left[\left( \widetilde{\Sigma}_{N(i),N(i)} \right)^{-1} \right]_{r,t}\right|\\ 
&\leq 1 - \gamma.
\end{align}
To conclude the proof of the claim, we will provide a lower bound on $\beta_{\rm min}^\ast$. Towards this end, note that by Claim~\ref{claim:betaStarIsRight}, we have that 
\begin{align}
\beta^\ast_{\rm min} \geq \min_{i,j\in [p]: K_{ij}\neq 0} \frac{\left|K_{ij}\right|}{\max\{\left| K_{ii} \right|. \left| K_{jj} \right|\}}\geq \frac{m}{\max_{i\in [p]}\left| K_{ii} \right|},\label{eq:betaMinPitstop}
\end{align}
where the last inequality follows from assumption~(A2). We will proceed by obtaining an upper bound on the denominator. Towards this end, we will employ the well Kantorovich inequality for positive definite matrices (see, e.g., \cite{horn2012matrix}). This inequality states that for a positive definite matrix $A\in \mathbb{R}^{d\times d}$ with real eigenvalues $L\leq \lambda_1\leq \cdots \leq \lambda_d\leq U$, the following holds 
\begin{align}
1 \leq \left(x^T A x \right) \left( x^T A^{-1} x \right) \leq \frac{1}{4}\left(\frac{L}{U} + \frac{U}{L} + 2\right), \;\;\mbox{for any $x\in \mathbb{R}^d$}. 
\end{align}
Therefore, choosing $x$ to be the $i-$th canonical vector $e_i$, we have the following useful inequality relating the diagonal elements of a matrix to the diagonal elements of its inverse 
\begin{align}
A^{-1}_{ii} \leq \frac{1}{4 A_{ii}}\left(\frac{L}{U} + \frac{U}{L} + 2\right).
\end{align}
Applying this to \eqref{eq:betaMinPitstop}, we get the desired result. This concludes the proof of the claim. 
\end{proof}

Claim~\ref{claim:covMatrixIsOkay} paves the way for applying Theorem~\ref{thm:lassoSupportRecovery} to our setting. Before, we conclude, we observe the following: 
\begin{enumerate}
\item[(a)] The noise variance ($\sigma^2$) of Theorem~\ref{thm:lassoSupportRecovery} can be taken to be $\max_{i\in [p]} \Sigma_{ii}$, from \eqref{eq:AMPLProofModel}. It is not hard to see that since $\Sigma$ is a positive definite matrix, we have that $\max_{i\in [p]} \Sigma_{ii} = \left\| \Sigma \right\|_\infty$, the absolute maximum element of $\Sigma$. 
\item[(b)] Our bound from Claim~\ref{claim:covMatrixIsOkay} on $\beta^\ast_{\rm min}$ implies that we can satisfy the so-called ``beta-min'' condition required by Theorem~\ref{thm:lassoSupportRecovery}, provided $m$ is as in the statement of Theorem~\ref{thm:lassoAlg}
\end{enumerate} 
Therefore, we now have that there exists constants $C_1, C_2 >0$ such that if we set $n_\ell  = C_1 \ell \log p$ and $\lambda_\ell = \sqrt{\frac{2C_2 \left\| \Sigma \right\|_\infty}{C_1 \ell \gamma^2}}$, we can be guaranteed that {\bf ndbSelect} succeeds with probability at least $1 - \delta/2pd_{\rm max}$. 

This concludes the proof. 
\end{proof}

\begin{lemma}
\label{lemma:AMPLC2Holds}
Fix an arbitrary $\ell \leq d_{\rm max}$, a vertex $i\in [p]$ and subsets $F, G \subseteq [p]$ that are such that $d_i \leq \ell$, $\overline{N}(i)\subseteq F$, and $H\subseteq F$. 	There exists a constant $C_4>0$ such that if we set $h(\ell) = C_4 \ell \log p$, then the probability that {\bf nbdVerify} fails (as in (C2)) is at most $\delta/2p d_{\rm max}$. 
\end{lemma}
The proof of this lemma follows directly from the proof of Theorem~\ref{thm:exhaustiveAlg}. Therefore, using Lemma~\ref{lemma:AMPLC1Holds} and Lemma~\ref{lemma:AMPLC2Holds} in Theorem~\ref{thm:metaAlg}, completes the proof of Theorem~\ref{thm:lassoAlg}. \end{proof}

\section{Helpful Results}
\label{sec:lemmata}
\subsection{Concentration of Partial Correlation Coefficients}
In this section, we will state a lemma that characterizes the concentration of empirical partial correlation coefficients (defined as in the {paragraph} after \eqref{eq:partialCorrelationRecursion}) about their expected values. See \cite{kalisch2007estimating} for a proof. 
\begin{lemma}
\label{lemma:partialCorrelationConcentration}
Provided (A2) holds, given $n$ samples from $(X_i,X_j,X_S)$, if the partial correlation coefficient $\widehat{\rho}_{i,j\mid S}$ is defined as above, then we have the following result
\begin{align}
\mathbb{P}\left[ \left| \widehat{\rho}_{i,j\mid S} - \rho_{i,j\mid S} \right| \geq \epsilon \right]\leq C_1 \left( n - 2 - \left| S \right| \right)exp\left\{ -\left( n - 4 - \left| S \right| \right)\log\left( \frac{4 + \epsilon^2}{4 - \epsilon^2} \right) \right\},
\end{align}
where $C_1 > 0$ is a constant that depends on $M$ from (A2). 
\end{lemma}

\subsection{Support Recovery for Lasso}
Assume that $y = X \beta^\ast + w$, where $y, w \in \mathbb{R}^n$, $\beta^\ast \in \mathbb{R}^p$, and $X\in \mathbb{R}^{n\times p}$  with iid rows $x_i \sim \mathcal{N}(0,\Sigma)$. Suppose $S$ is the support of $\beta^\ast$ and suppose that the following hold 
\begin{align}
\vertiii{\Sigma_{S^c S} \left( \Sigma_{SS} \right)^{-1}}_\infty &\leq 1 - \gamma, \gamma\in (0,1]\label{eq:irrepresentability}\\
\Lambda_{\rm min}\left(\Sigma_{SS}\right) &\geq C_{\rm min} > 0\\
\Lambda_{\rm max}\left(\Sigma_{SS}\right) &\leq C_{\rm min} < +\infty
\end{align}
If we let $\widehat{\beta}\in \mathbb{R}^p$ denote the solution to the Lasso problem 
\begin{align}
\widehat{\beta} \triangleq \frac{1}{2n}\min_{\beta\in \mathbb{R}^p} \left\| y - X\beta \right\|_2^2 + \lambda_n \left\| \beta \right\|_1, 
\end{align}
then we have the following result. 
\newtheorem{exe}{Theorem}
\setcounter{exe}{3}

\begin{exe}
\label{thm:lassoSupportRecovery}
Suppose $w\sim \mathcal{N}(0,\sigma^2 I)$ and suppose that $\Sigma$ satisfies the properties listed above. Then, there exists constants $C_1,C_2,C_3, C_4, C_5$ such that if $\lambda_n = \sigma\gamma^{-1}\sqrt{2 C_1 \log p/n}$, $n \geq C_2 k \log p$, and $\beta_{\rm min} \triangleq \min_{i\in S}\left| \beta_i^\ast \right| > u(\lambda_n)$, where
\begin{equation}
u(\lambda_n) \triangleq C_5 \lambda_n \vertiii{\Sigma_{SS}^{-1/2}}^2_\infty + 20 \sqrt{\frac{\sigma^2 \log k}{C_{\rm min} n}},
\end{equation}
the support $\widehat{\beta}$ is identical to that of $\beta^\ast$ with probability exceeding $1 - C_3 p^{-C_4}$. 
\end{exe}
\begin{proof}
The proof of this theorem follows almost entirely from Theorem 3 in \cite{wainwright2009sharp}. In fact, the only thing we modify from that result is the rate of decay of the probability of error. In particular, we will show that the probability of error decays polynomially in $p$ (or equivalently exponentially in $\log p$) for all values of $k$, whereas Theorem 3 of \cite{wainwright2009sharp} shows that the error decays exponentially in $\min\left\{ k,\log (p-k) \right\}$, which is somewhat weak for our purposes. 

Towards this end, it is not hard to see that the result that requires strengthening is Lemma~5 in \cite{wainwright2009sharp}. We furnish a sharper substitute in Lemma~\ref{lemma:waiLemma5}. 
\end{proof}

\begin{lemma}
\label{lemma:waiLemma5}
Consider a fixed $z\in \mathbb{R}^k$, a constant $c_1>0$, and a random matrix $W\in \mathbb{R}^{n\times k}$ with i.i.d elements $W_{ij}\sim \mathcal{N}\left( 0,1 \right)$. Suppose that $n \geq \max\left\{\frac{4}{\left( \sqrt{8} - 1 \right)^2}k, \frac{64}{c_1^2}k \log (p-k) \right\}$, then there exists a constant $c_2>0$ such that 
\begin{align*}
&\mathbb{P}\left[ \left\| \left[\left( \frac{1}{n}W^T W \right)^{-1} - I_k\right]z \right\|_\infty \geq C_1 \left\| z \right\|_\infty\right]\nonumber\\
&\qquad\qquad\qquad\leq 4 \exp\left( -c_2 \log(p-k) \right)
\end{align*}
\end{lemma}
\begin{proof}
Set $A = \left( \frac{1}{n} W^T W \right)^{-1} - I_{k}$. Observe that $\mathbb{P}\left[ \left\| Az \right\|_\infty \geq c_1 \left\| z \right\|_\infty\right]\leq \mathbb{P}\left[ \left\|A\right\|_\infty \geq c_1\right]$ by the definition of the matrix infinity norm. Next, observe that since the infinity norm is the maximum absolute row sum of the matrix, we have that $\mathbb{P}\left[ \left\| A \right\|_\infty\geq c_1  \right]\leq \mathbb{P}\left[ \left\| A \right\|_2\geq c_1/\sqrt{k}  \right]$. From \cite[Lemma 9]{wainwright2009sharp} (which follows in a straightforward manner from the seminal results of \cite{davidson2001local}), we know that 
\begin{align}
\mathbb{P}\left[ \left\| A \right\|_2 \geq \delta(n,k,t)\right] \leq 2 e^{-nt^2/2},\label{eq:DSinWaiLemma}
\end{align}
where $\delta(n,k,t) = 2\left( \sqrt{\frac{k}{n}} + t \right) + \left( \sqrt{\frac{k}{n}}+t \right)^2$. We will divide the proof into three cases: 

\noindent\underline{\bf Case (a): $k\leq \frac{c_1^2}{64}$}\\
Suppose we pick $t = \sqrt{\frac{c_1}{\sqrt{k}}}-1 - \sqrt{\frac{k}{n}}$, under the setting of this case, provided $n\geq \frac{4}{\left( \sqrt{8}-1 \right)^2}k$, we have that $t > \frac{\sqrt{8} - 1}{2}>0$. Notice that for this choice of $t$, we have $\delta(n,k,t) = \frac{c_1}{\sqrt{k}}$. This gives us the following bound 
\begin{align}
\mathbb{P}\left[ \left\| A \right\|_2 \geq \frac{c_1}{\sqrt{k}} \right]&\leq 2 \exp\left\{-\frac{n}{2}\left({\sqrt{\frac{{c_1}}{\sqrt{k}}}} -1 - \sqrt{\frac{k}{n}}\right)^2\right\}\\
&\leq 2 \exp\left\{ -\frac{n}{2}\left( \frac{\sqrt{8} - 1}{2} \right)^2 \right\}
\end{align}

\noindent\underline{\bf Case (b): $\log (p-k)\geq k> \frac{c_1^2}{64}$}\\
Suppose we pick $t = \frac{c_1}{8\sqrt{k}}$, we have that $t <1$, by the assumption of this case. Then, if $n \geq \frac{64 k^2}{c_1^2}$ observe that 
\begin{align}
\delta(n,k,t) &= 2 \left( \sqrt{\frac{k}{n}} + t \right) + \left( \sqrt{\frac{k}{n}} + t \right)^2\\
&\leq \frac{c_1}{\sqrt{k}}.
\end{align}
This implies that 
\begin{align}
\mathbb{P}\left[ \left\| A \right\|_2 \geq \frac{c_1}{\sqrt{k}} \right]\leq 2 \exp\left\{ - \frac{nc_1^2}{128 k}\right\}.
\end{align}
Notice that if $n \geq \frac{64}{c_1^2}k\log(p-k)$, then $n \geq \frac{64 k^2}{c_1^2}$, as required. 

\noindent\underline{\bf Case (c): $k > \log(p-k)$}\\
In this case, we can adopt the result from Lemma 5 of \cite{wainwright2009sharp}. 

Putting all this together, we get the desired result.  
\end{proof}

\end{document}